\documentclass{article} 
\usepackage{NIPS2015,times}
\usepackage{hyperref}
\usepackage{url}

\usepackage{graphicx}
\usepackage{braket}
\usepackage{amsmath}
\usepackage{amssymb}
\usepackage{amsthm}
\usepackage{bbm}
\usepackage{subcaption}
\usepackage{caption}

\DeclareMathOperator*{\argmax}{arg\,max}
\usepackage{multicol}
\usepackage{algorithm}
\usepackage{algpseudocode}
\algrenewcommand\textproc{\texttt}

\newtheorem{theorem}{Theorem}

\newtheorem{proposition}{Proposition}
\newtheorem{lemma}{Lemma}

\title{Searching for significant patterns in stratified data} 

\author{
Felipe Llinares-L\'{o}pez\thanks{These authors contributed equally to this work. Corresponding e-mail addresses: \newline  \texttt{felipe.llinares@bsse.ethz.ch},  \texttt{laetitia.papaxanthos@bsse.ethz.ch}, \newline \texttt{dean.bodenham@bsse.ethz.ch}}\\
D-BSSE, ETH Z\"{u}rich\\
\And
Laetitia Papaxanthos$^{*}$ \\
D-BSSE, ETH Z\"{u}rich\\
\And
Dean Bodenham$^{*}$ \\
D-BSSE, ETH Z\"{u}rich\\
\AND
Karsten Borgwardt \\
D-BSSE, ETH Z\"{u}rich\\
}

\nipsfinalcopy 

\begin{document}

\maketitle

\begin{abstract}
Significant pattern mining, the problem of finding itemsets that are significantly enriched in one class of objects, 
is statistically challenging, as the large space of candidate patterns leads to an enormous multiple testing problem. Recently, the concept of {\it testability} was proposed as one approach to correct for multiple testing in pattern mining while retaining statistical power. Still, these strategies based on testability do not allow one to condition the test of significance on the observed covariates, which severely limits its utility in biomedical applications. Here we propose a strategy and an efficient algorithm to perform significant pattern mining in the presence of categorical covariates with $K$ states.


\end{abstract}

\section{Introduction}

The goal of \textit{frequent pattern mining}~\cite{Aggarwal14FPM} is to find patterns in a dataset of objects that occur in at least $\theta$ percent of all objects. Its most popular instance is \textit{frequent itemset mining}, in which one is given a collection of sets (transactions) and tries to find subsets of elements (items) that frequently co-occur in the same sets. The classic example is \textit{shopping basket analysis}, in which one tries to find groups of products that are frequently co-bought by the customers of a supermarket~\cite{Agrawal93}.

\textit{Significant} (or \textit{discriminative} or \textit{contrast}) \textit{pattern mining}~\cite{Dong12} does not focus on the absolute frequency of a pattern in a dataset, but rather whether it is statistically significantly enriched in one of two classes of objects. 
Finding groups of mutations that are statistically significantly more common in disease-carriers than healthy controls is one of the many biomedical applications of significant pattern mining~\cite{fang_high-order_2012}.

One of the biggest problems in pattern mining is the fact that the number of candidate patterns grows exponentially with the number of items in a pattern. The resulting computational problem of enumerating frequent patterns has been intensively investigated for decades~\cite{Agrawal93}. The accompanying statistical problem, that each candidate pattern consists of a hypothesis in significant pattern mining and that testing millions and billions of patterns for significance leads to an enormous multiple testing problem, has received far less attention. 

Only recently, Terada et al.~\cite{TeradaPNAS} introduced the concept of testability, originally developed by Tarone~\cite{Tarone}, to pattern mining: the key idea is that if we are working with discrete test statistics, e.g. Fisher's exact test, to quantify whether a pattern is statistically significantly enriched in one class, then many patterns are too rare or too frequent to ever become significantly enriched. Tarone argued that these {\it untestable hypotheses} (or untestable patterns) need not be considered when correcting for multiple testing (e.g. via Bonferroni correction~\cite{Bonferroni36}), thereby greatly improving the statistical power in detecting truly significant patterns.

While this new approach is of great importance for high-dimensional data mining, it is limited in one crucial aspect: It cannot account for the effect of covariates, which are omnipresent in particular biomedical applications. For instance, the enrichment of particular mutations in disease-carriers may heavily depend on a person's age group, gender, or the geographic subpopulation to which he or she belongs. It has been long appreciated by the Genetics community~\cite{price_principal_2006} that ignoring these covariates leads to a surge in false positive patterns. Hence, significant pattern mining methods which do not account for these covariates have limited application in this domain. 

{\it In order to bridge this gap, our goal in this article is to enable significant pattern mining in data that is stratified according to a categorical variable with $K$ states.}

To this end, we propose a strategy based on the Cochran-Mantel-Haenszel test for $K$ 2 $\times$ 2 contingency tables~\cite{mantel1959} and increase its statistical power by pruning non-testable patterns (Section~\ref{sec:background}). As this approach scales exponentially with $K$ we also propose an exact algorithm to perform the test in a runtime which is only of order $O(K \log K)$ (Section~\ref{sec:proposal}). Experiments on simulated and real data demonstrate that our approach is superior in terms of runtime, statistical power and false positive rate when compared to all competing methods (Section~\ref{sec:exp}).
We believe that our approach opens the door to numerous applications of significant pattern mining.


\section{Significant pattern mining and the multiple testing problem}
\label{sec:background}

\subsection{Significant pattern mining}
\label{sec:significant_pattern_mining}

Let $\mathcal{E}=\set{e_{1},\hdots,e_{m}}$ be a universe of $m$ \emph{items}. We call any set $\mathcal{S} \subseteq \mathcal{E}$ containing $|\mathcal{S}|$ distinct items a \emph{pattern} and assume we are given a database $\mathcal{D} = \set{\left(\mathcal{T}_{i},y_{i}\right)}_{i=1}^{n}$ containing $n$ \emph{transactions} $\mathcal{T}_{i} \subseteq \mathcal{E}$ with binary labels $y_{i} \in \set{0,1}$. Informally, the goal of significant pattern mining is to find all patterns $\mathcal{S}$ whose occurrence within a transaction $\mathcal{T}$ is informative of the label $y$ of $\mathcal{T}$. 

Given a pattern $\mathcal{S}$, let $G(\mathcal{S},\mathcal{T})=\mathbbm{1}[\mathcal{S} \subseteq \mathcal{T}]$ be an indicator binary random variable which takes value $1$ whenever $\mathcal{S} \subseteq \mathcal{T}$ and value $0$ otherwise. The pattern $\mathcal{S}$ is \emph{statistically associated} with the class labels $y$ if $G(\mathcal{S},\mathcal{T})$ and the class labels $y$ are statistically dependent random variables. Given a transaction database $\mathcal{D}$ containing $n$ i.i.d. labeled transactions, we can estimate the statistical dependence between $G(\mathcal{S},\mathcal{T})$ and the class labels $y$ by choosing an appropriate \emph{test statistic} and computing the corresponding \emph{$p$-value}. The $p$-value of the observed association between $G(\mathcal{S},\mathcal{T})$ and the class labels $y$ in $\mathcal{D}$ is defined as the probability of obtaining a value for the test statistic at least as extreme as the one observed in $\mathcal{D}$ under the \emph{null hypothesis} of statistical independence between $G(\mathcal{S},\mathcal{T})$ and $y$. A low $p$-value indicates that it is unlikely that the database $\mathcal{D}$ was generated by a model in which the random variables $G(\mathcal{S},\mathcal{T})$ and $y$ are statistically independent. Thus, the pattern $\mathcal{S}$ and the class labels $y$ are deemed to be statistically significantly associated if the $p$-value is below a predefined {\it significance threshold} $\alpha$. From the definition of $p$-value it follows that the probability of finding an inexistent association between pattern $\mathcal{S}$ and the class labels $y$, i.e. having a \emph{false positive}, is upper bounded by $\alpha$. Therefore, the significance threshold controls the trade-off between the probability of false positives and the \emph{statistical power} to discover truly associated patterns.

Since both $G(\mathcal{S},\mathcal{T})$ and $y$ are binary random variables, one can summarise the relevant information from the observed database $\mathcal{D}$ as a $2 \times 2$ contingency table:

\begin{center}
	\begin{tabular}{|c|c|c|c|}
		\hline
		Variables & $G(\mathcal{S},\mathcal{T})=1$ & $G(\mathcal{S},\mathcal{T})=0$ & Row totals\\
		\hline
		$y=1$ & $a_{\mathcal{S}}$ & $n_1-a_{\mathcal{S}}$ & $n_1$ \\
		\hline
		$y=0$ & $x_{\mathcal{S}}-a_{\mathcal{S}}$ & $(n-n_1)-(x_{\mathcal{S}}-a_{\mathcal{S}})$ & $n-n_1$ \\
		\hline
		Col totals & $x_{\mathcal{S}}$ & $n-x_{\mathcal{S}}$ & $n$ \\
		\hline
	\end{tabular}
\end{center}

$n$ is the number of transactions, $n_{1}$ of which have class label $y=1$; $x_{\mathcal{S}}$ is the number of transactions containing pattern $\mathcal{S}$ and $a_{\mathcal{S}}$ is the number of transactions with class label $y=1$ containing pattern $\mathcal{S}$. Popular methods to obtain $p$-values from $2 \times 2$ contingency tables include Fisher's exact test~\cite{FisherExactTest} and Pearson's $\chi^{2}$-test~\cite{Pearson1900}, among others.

In large-scale significant pattern mining, we must compute a $p$-value as described above for every pattern $\mathcal{S}$ in a potentially huge search space of patterns, i.e. for all $\mathcal{S} \in \mathcal{M}$, $\mathcal{M} \subseteq 2^{\mathcal{E}}$. Note that this formulation is very general and includes classical significant pattern mining problems such as significant itemset mining, significant subgraph mining or significant interval search as particular instances. As we will discuss in section~\ref{sec:mult_hyp_test}, the overwhelming size of the search space poses both computational and statistical challenges that can only be overcome by combining special statistical testing methods for discrete data with efficient branch-and-bound algorithms.

\subsection{Significant pattern mining in stratified datasets}

Suppose we have a collection $\set{\mathcal{D}}_{k=1}^{K}$ of $K$ transaction databases defined under the same universe of items $\mathcal{E}$. This situation arises naturally when the data is influenced by an observed categorical covariate of cardinality $K$.  
In this setting, one can perform a meta-analysis and construct a separate $2 \times 2$ contingency table for each of the $K$ values of the categorical covariate to then combine the individual results.

One of the first meta-analysis procedures is Fisher's combined probability test~\cite{Fisher1925}, which computes a $p$-value for each of the $K$ $2 \times 2$ contingency tables independently and then combines them as $p_{\mathrm{fisher}}=-2\sum_{i=1}^{K}{p^{i}}$. It can be shown that under the global null hypothesis that there is no statistical association in any of the $K$ contingency tables, $p_{\mathrm{fisher}} \sim \chi^{2}_{2k}$. Another popular approach is the Cochran-Mantel-Hanszel (CMH) test~\cite{mantel1959}. The CMH test statistic is defined directly as a function of the cell counts of each of the $K$ contingency tables as:

\begin{equation}
\label{eq:cmh_def}
	T_{\mathrm{cmh}} = \frac{\left(\sum_{i=1}^{K}{a_{\mathcal{S}}^{i} - \gamma^{i}x_{\mathcal{S}}^{i}}\right)^{2}}{\sum_{i=1}^{K}{\gamma^{i}(1-\gamma^{i})x_{\mathcal{S}}^{i}\left( 1-\frac{x_{\mathcal{S}}^{i}}{n^{i}}\right) }}
\end{equation}

where $\gamma^{i}=n_{1}^{i} / n^{i}$. The CMH statistic defined in equation~\eqref{eq:cmh_def} can be seen as an extension of Pearson's $\chi^{2}$-test to meta-analysis, as both are identical for $K=1$. Moreover, it can be shown~\cite{mantel1959} that under the global null hypothesis $T_{\mathrm{cmh}}$ also follows a $\chi^{2}_{1}$ distribution for any $K$. Unlike other meta-analysis methods, such as Fisher's combined probability test, the CMH statistic combines the statistical association found across databases taking the sign of the correlation into account as well. 
In this article, we will choose the CMH statistic as the meta-analysis procedure but our work can be readily extended to other methods just as Fisher's combined probability test.


\subsection{Multiple testing problem in pattern mining}
\label{sec:mult_hyp_test}

The main statistical challenge in large-scale significant pattern mining is the so called \emph{multiple comparisons problem}. As discussed in section~\ref{sec:significant_pattern_mining}, the search space of patterns $\mathcal{M} \subseteq 2^{\mathcal{E}}$ is usually extremely large, with $\mathcal{M}$ in the order of billions or even trillions. Thus, if the occurrence within transactions of every pattern $\mathcal{S} \in \mathcal{M}$ was tested for association with the class labels at level $\alpha$, the overall expected number of false positives would be approximately $\alpha |\mathcal{M}|$. For classical choices of the significance threshold, such as $\alpha=0.05$, one would obtain from hundreds of millions to billions of false positives, rendering the procedure not applicable in practice. Therefore, rather than controlling the probability of having false positives for every pattern $\mathcal{S}$ individually, it would be desirable to instead control the FWER (Family Wise Error Rate), defined as the probability of producing one or more false positives when exploring the whole search space $\mathcal{M}$. That can be achieved by computing a  \emph{corrected significance threshold} $\delta$ such that if every $\mathcal{S} \in \mathcal{M}$ is deemed statistically significantly associated when $p_{\mathcal{S}} \le \delta$, one has $\mathrm{FWER} \le \alpha$. 
The most popular way to correct for multiple testing is to use a \emph{Bonferroni correction}~\cite{Bonferroni36}. 
By upper bounding $\mathrm{FWER}(\delta) \le \delta \mathcal{M}$ it follows that $\delta^{*}_{\mathrm{bon}}=\alpha / \mathcal{M}$ controls the FWER. This classic procedure is known to be very conservative and to largely lose statistical power if $|\mathcal{M}|$ is huge as in pattern mining.


Tarone was first to point out in~\cite{Tarone} that the discrete nature of $2 \times 2$ contingency tables can be used to compute a lower bound on the attainable $p$-values and obtaine an improved corrected significance threshold. For a given $2 \times 2$ contingency table with fixed margins $x_{\mathcal{S}}$, $n_{1}$ and $n$ one can show that $a_{\mathcal{S}}\in [\![a_{\mathcal{S},min},a_{\mathcal{S},max}]\!]$ with $a_{\mathcal{S},min}=\max(0,x_{\mathcal{S}}-(n-n_1))$ and $a_{\mathcal{S},max}=\min(x_{\mathcal{S}},n_1)$. Thus, there are at most $a_{\mathcal{S},max}-a_{\mathcal{S},min}+1$ different values that the $p$-value could take. The \emph{minimum attainable $p$-value} is then $\Psi(x_\mathcal{S})= \min\set{p_{\mathcal{S}}(k) | a_{\mathcal{S},\mathrm{min}} \leq \gamma \leq a_{\mathcal{S},\mathrm{max}}}$. By definition, only the patterns belonging to the \emph{set of testable patterns at corrected significance level $\delta$}, $\mathcal{I}_{T}(\delta) = \{\, \mathcal{S} \mid \Psi(x_{\mathcal{S}}) \le \delta\,\}$, can be statistically significantly associated at level $\delta$. If the test statistic of choice considers the table margins fixed\footnote{This holds for most popular test statistics for $2 \times 2$ contingency tables, such as Fisher's Exact Test or Pearson's $\chi^{2}$-test.} then patterns $\mathcal{S} \in \mathcal{M} \setminus \mathcal{I}_{T}(\delta)$ do not affect the FWER at level $\delta$. Thus, one can obtain an improved corrected significance threshold as:
\begin{align}
\label{eq:tarone_threshold}
	\delta^{*}_{\mathrm{tar}} = \max\left\{\delta \, | \, \delta\left\vert \mathcal{I}_{T}(\delta) \right\vert \le \alpha\right\}
\end{align}

How to compute this number of testable patterns for improved multiple testing correction has been the focus of several recent studies~\cite{TeradaPNAS,Sugiyama15SDM,TeradaIEEE}. These approaches exploit a link the between the frequency of a pattern and its testability. This connection does not hold in meta-analysis (see Section~\ref{sec:algorithm}, Paragraph `Search Space Pruning'), which therefore requires the development of new efficient techniques to detect all testable patterns in meta-analyses.  
In the next section, we propose such a novel efficient algorithm, \texttt{FastCMH}, 
which can find patterns $\mathcal{S}$ associated with the class labels $y$ conditioned on a categorical covariate by using the CMH test, while only considering testable patterns in its correction for multiple testing.

\section{Efficient exact computation of $T_{\mathrm{cmh}}$: FastCMH}
\label{sec:proposal}

In this section we our main contribution, the \texttt{FastCMH} algorithm, in detail. In a nutshell, FastCMH is an efficient algorithm to solve equation~\eqref{eq:tarone_threshold} when the CMH statistic is used to perform meta-analysis across $K$ transaction databases. Since equation~\eqref{eq:tarone_threshold} essentially poses a one-dimensional optimisation problem, the solution $\delta^{*}_{\mathrm{tar}}$ can be found via line search. However, the main difficulty stems from the fact that evaluating $\delta\left\vert \mathcal{I}_{T}(\delta) \right\vert$ exactly is extremely costly and, therefore, computational tricks are required to tackle the optimisation problem. We first describe the pseudocode of FastCMH in subsection \ref{sec:algorithm} and then its theoretical properties. 

\subsection{The Algorithm: FastCMH}
\label{sec:algorithm}

The pseudocode of \texttt{FastCMH} is shown in Algorithm \ref{alg:fastcmh}:

\begin{algorithm}[ht]
\begin{small}
\caption{FastCMH}\label{alg:fastcmh}
\begin{multicols}{2}
 \begin{algorithmic}[1]
  \State {\bfseries Input:} $\set{\mathcal{D}}_{k=1}^{K}$ and $\alpha$
  \State {\bfseries Output:} Corrected significance threshold $\delta^*$
	\Function{FastCMH}{$\alpha$, $\set{\mathcal{D}}_{k=1}^{K}$}	
		\State $j \leftarrow 0$, $\delta \leftarrow 10^{-j\mu}$ \label{algline:init_begin}
		\State $\mathtt{buckets}[j] \leftarrow 0 \; \forall j=0,\hdots,N_{\mathrm{steps}}-1$
		\State $\left\vert\mathcal{I}_{T}(\delta)\right\vert \leftarrow 0$ \label{algline:init_end}
 		\State \texttt{process\_next}$(\mathrm{root}, \set{n_{1}^{i}}_{i=1}^{K})$ \label{algline:start_enum}
        		\State Return $\delta^{*}_{\mathrm{tar}}= \alpha / \left\vert\mathcal{I}_{T}(\delta)\right\vert$
	\EndFunction
	\Function{process\_next}{$\mathcal{S}$, $\set{x_{\mathcal{S}}^{i}}_{i=1}^{K}$}
			\If{$\Psi(\set{x_{\mathcal{S}}^{i}}_{i=1}^{K}) \le \delta$} \label{algline:check_testability}
				\State $\left\vert\mathcal{I}_{T}(\delta)\right\vert \leftarrow \left\vert\mathcal{I}_{T}(\delta)\right\vert+1 $ \label{algline:testable_block_begin}
				\State \texttt{update\_bucket}($\Psi(\set{x_{\mathcal{S}}^{i}}_{i=1}^{K})$) 			 
				 \While{$\delta\left\vert\mathcal{I}_{T}(\delta)\right\vert > \alpha$}
				 	\State $\left\vert\mathcal{I}_{T}(\delta)\right\vert \leftarrow \left\vert\mathcal{I}_{T}(\delta)\right\vert - \mathtt{buckets}[j]$ \label{algline:update_testable_pattern_count}
					\State $j \leftarrow j+1$, $\delta \leftarrow 10^{-j\mu}$ 
				 \EndWhile  \label{algline:testable_block_end}
			\EndIf
			\If{\texttt{is\_not\_prunable}($\set{x_{\mathcal{S}}^{i}}_{i=1}^{K}$,$\delta$)}  \label{algline:evaluate_pruning}
				\For{$\mathcal{S}^{\prime} \in \mathrm{Children}(\mathcal{S})$} \label{algline:cont_enum_begin}
					\State Compute $x_{\mathcal{S}^{\prime}}^{i}  \; \forall i=1,\hdots,K$
					\State \texttt{process\_next}($\mathcal{S}^{\prime}$,$x_{\mathcal{S}^{\prime}}$)
				\EndFor \label{algline:cont_enum_end}
			\EndIf
	\EndFunction
	\Function{update\_bucket}{$p$}
		\State $i \leftarrow \lfloor -\log_{10}(p) / \mu\rfloor$
		\State $\mathtt{buckets}[i] \leftarrow \mathtt{buckets}[i] + 1$
	\EndFunction
	\Function{is\_not\_prunable}{$\set{x_{\mathcal{S}}^{i}}_{i=1}^{K}$,$\delta$}
    	\If{$\exists i \; | \; x_{\mathcal{S}}^{i} > \min(n_{1}^{i},n^{i}-n_{1}^{i})$}
        	\State Return True
        \EndIf
		\State $T_{\mathrm{prune}} \leftarrow \underset{\set{0 \le z^{i} \le x_{\mathcal{S}}^{i}}_{i=1}^{K}}{\max}T_{\mathrm{cmh}}^{\mathrm{max}}(\set{z^{i}}_{i=1}^{K})$ 
		\State $\Psi_{\mathrm{prune}} \leftarrow 1 - F_{\chi^{2}_{1}}(T_{\mathrm{prune}})$
		\State Return $\Psi_{\mathrm{prune}} \le \delta$ 
	\EndFunction
 \end{algorithmic}
 \end{multicols}
 \end{small}
 \end{algorithm}
 
\textbf{Initialization:} Between lines~\ref{algline:init_begin} and~\ref{algline:init_end} the main variables of CMH are initialized. We try to find $\delta^{*}_{\mathrm{tar}}$ by sampling $N_{\mathrm{steps}}$ tentative values for $\delta^{*}_{\mathrm{tar}}$ in a logarithmic grid $10^{-j\mu}$ with step size $\mu$ in the exponent. Due to the discrete nature of the problem, as long as $\mu$ is small enough and $N_{\mathrm{steps}}$ large enough the method is insensitive to those parameters. We used $\mu=0.06$ and $N_{\mathrm{steps}}=500$ throughout the article. We initialise the tentative solution $\delta$ to $1$, the largest value in the grid. Also, we create an array $\mathtt{buckets}$, initialised to $0$, such that $\mathtt{buckets}[i]$ equals the number of patterns $\mathcal{S}$ found so far with minimum attainable $p$-value $10^{-(i+1)\mu} \le \Psi(\set{x_{\mathcal{S}}^{i}}_{i=1}^{K}) \le 10^{-i\mu}$. We also initialise the count of patterns $\mathcal{S}$ explored so far with $\Psi(\set{x_{\mathcal{S}}^{i}}_{i=1}^{K}) \le \delta$, $\left\vert\mathcal{I}_{T}(\delta)\right\vert$, to $0$. After the initialisation, in line~\ref{algline:start_enum} the enumeration process beings.

\subsubsection*{FastCMH core: The \texttt{process\_next} function}

{\bf Testability:} We assume that the search space $\mathcal{M}$ can be arranged in the shape of a tree, with patterns $\mathcal{S} \in \mathcal{M}$ corresponding to nodes in such a way that children $\mathcal{S}^{\prime}$ of a pattern . This assumption covers a vast amount of classical data mining problems such as itemset and subgraph mining. The function \texttt{process\_next} is the core of CMH and is devoted to analysing a given pattern $\mathcal{S}$ and iteratively exploring the search space $\mathcal{M}$ in a depth-first manner. 

First of all, in line~\ref{algline:check_testability} the algorithm checks if pattern $\mathcal{S}$ is testable at the current level $\delta$, i.e. if $\Psi(\set{x_{\mathcal{S}}^{i}}_{i=1}^{K}) \le \delta$. That requires extending Tarone's method to the CMH test, which amounts to finding
\begin{align}
	T_{\mathrm{cmh}}^{\mathrm{max}}(\set{x_{\mathcal{S}}^{i}}_{i=1}^{K}) &= \underset{\set{a_{\mathcal{S},min}^{i} \le a_{\mathcal{S}}^{i} \le a_{\mathcal{S},max}^{i}}_{i=1}^{K}}{\max}T_{\mathrm{cmh}}^{\mathrm{max}}(\set{a_{\mathcal{S}}^{i}}_{i=1}^{K}) \\
	\Psi(\set{x_{\mathcal{S}}^{i}}_{i=1}^{K}) &= 1-F_{\chi^{2}_{1}}(T_{\mathrm{cmh}}^{\mathrm{max}}(\set{x_{\mathcal{S}}^{i}}_{i=1}^{K}))
\end{align}
where we have omitted the dependence on $\set{n_{1}^{i}}_{i=1}^{K}$ and $\set{n^{i}}_{i=1}^{K}$ to simplify the notation and $F_{\chi^{2}_{1}}(x)$ denotes the distribution function of a $\chi^{2}_{1}$ random variable. In the Appendix we prove the following proposition:
\begin{proposition}
\label{prop:tcmh_max}
	\begin{equation}
	T_{\mathrm{cmh}}^{\mathrm{max}}(\set{x_{\mathcal{S}}^{i}}_{i=1}^{K}) = \frac{\max\left[ \left( \sum_{i=1}^{K}{a_{\mathcal{S},min}^{i} - \gamma^{i}x_{\mathcal{S}}^{i}} \right)^{2}   , \left( \sum_{i=1}^{K}{a_{\mathcal{S},max}^{i} - \gamma^{i}x_{\mathcal{S}}^{i}}\right)^{2} \right]}{\sum_{i=1}^{K}{\gamma^{i}(1-\gamma^{i})x_{\mathcal{S}}^{i}\left( 1-\frac{x_{\mathcal{S}}^{i}}{n^{i}}\right) }}
	\end{equation}
\end{proposition}

The result in proposition~\ref{prop:tcmh_max} allows us to compute $\Psi(\set{x_{\mathcal{S}}^{i}}_{i=1}^{K})$ exactly in $O(K)$ time. If the pattern is testable at the current level $\delta$, the block of code between lines~\ref{algline:testable_block_begin} and \ref{algline:testable_block_end} is executed. Firstly, the number of testable patterns processed so far, $\left\vert\mathcal{I}_{T}(\delta)\right\vert$, and the corresponding entry of the array $\mathtt{buckets}$ are increased to account for the newly processed testable pattern. Then, the FWER condition $\delta\left\vert \mathcal{I}_{T}(\delta) \right\vert \le \alpha$ is checked. If not satisfied, then we must decrease the tentative corrected significance threshold $\delta$ until the condition is satisfied again. Notice that every time $\delta$ is decreased to the next point in the grid, the patterns $\mathcal{S}$ for which $10^{-(j+1)\mu} \le \Psi(\set{x_{\mathcal{S}}^{i}}_{i=1}^{K}) \le 10^{-j\mu}$ become non-testable. Therefore, $\left\vert\mathcal{I}_{T}(\delta)\right\vert$ is adjusted by subtracting $\mathtt{buckets}[j]$ in line~\ref{algline:update_testable_pattern_count}.

{\bf Search space pruning:} Next, one must determine whether children $\mathcal{S}^{\prime}$ of the current pattern $\mathcal{S}$ must be visited or whether we can prune the search space. This step is crucial since, otherwise, we would need to explore the entire search space $\mathcal{M}$, resulting in an inadmissible computational complexity. Since, by assumption, $x_{\mathcal{S}^{\prime}} \le x_{\mathcal{S}}$ for all children $\mathcal{S}^{\prime}$ of pattern $\mathcal{S}$, when the function $\Psi(x_{\mathcal{S}})$ is monotonically decreasing in $x_{\mathcal{S}} \in [0,\min(n_{1},n-n_{1})]$, one can conclude that children of non-testable patterns with $x_{\mathcal{S}} \le ,\min(n_{1},n-n_{1})$ must be necessarily non-testable as well. This is the key property that has been exploited by all existing previous work in significant pattern mining based on Tarone's method. Nonetheless, the function $\Psi(\set{x_{\mathcal{S}}^{i}}_{i=1}^{K})$ does not exhibit this monotonicity property:
\begin{proposition}
\label{prop:non_monotonicity}
	The minimum attainable $p$-value function for the CMH statistic, $\Psi(\set{x_{\mathcal{S}}^{i}}_{i=1}^{K})$, is not monotonically decreasing with respect to each $x_{\mathcal{S}}^{i}$ in $x_{\mathcal{S}}^{i} \in [0,\min(n_{1}^{i},n^{i}-n_{1}^{i})]$ if there exist $i$,$j$ in $1,\hdots,K$ such that $\frac{n_{1}^{i}}{n^{i}} < \frac{1}{2},\frac{n_{1}^{j}}{n^{j}} > \frac{1}{2}$
\end{proposition}
The proof of Proposition~\ref{prop:non_monotonicity} can be found in the Appendix. In the light of Proposition~\ref{prop:non_monotonicity}, the case of significant pattern mining for meta-analysis appears to be hopeless. However, as we will show next, it is possible to generalize the existing Tarone-based significant pattern mining framework beyond the requirements of monotonicity in $\Psi(x_{\mathcal{S}})$. The key idea is summarised in the following proposition:
\begin{proposition}
\label{prop:pruning_definition}
	Children $\mathcal{S}^{\prime}$ of a pattern $\mathcal{S}$ can be pruned from the search space if and only if $\Psi_{\mathrm{prune}} > \delta$ where:
	\begin{align}
		T_{\mathrm{prune}} = \underset{\set{0 \le z^{i} \le x_{\mathcal{S}}^{i}}_{i=1}^{K}}{\max}T_{\mathrm{cmh}}^{\mathrm{max}}(\set{z^{i}}_{i=1}^{K}) \\
		\Psi_{\mathrm{prune}} = 1 - F_{\chi^{2}_{1}}(T_{\mathrm{prune}})
	\end{align}
\end{proposition}
\begin{proof}
	Since children $\mathcal{S}^{\prime}$ of pattern $\mathcal{S}$ satisfy $\mathcal{S}$ satisfy $x_{\mathcal{S}^{\prime}} \le x_{\mathcal{S}}$, it follows immediately from the definition of $T_{\mathrm{prune}}$ that $\Psi(\set{x_{\mathcal{S}^{\prime}}^{i}}_{i=1}^{K}) \ge \Psi_{\mathrm{prune}} > \delta$ and, therefore, $\mathcal{S}^{\prime}$ must be non-testable.
\end{proof}

\begin{algorithm}[ht]
\begin{small}
\caption{Fast evaluation of $T_{\mathrm{prune}}$}\label{alg:eval_max_corner}
 \begin{algorithmic}[1]
 	\Function{eval\_T\_prune}{$\set{x_{\mathcal{S}}^{i}}_{i=1}^{K}$}
		\State $\beta_{l}^{i} \leftarrow (1-\gamma^{i})(1-\frac{x_{\mathcal{S}}^{i}}{n^{i}})$, $\beta_{r}^{i} \leftarrow (1-\gamma^{i})(1-\frac{x_{\mathcal{S}}^{i}}{n^{i}})$ for all $i=1,\hdots,K$
        \State $\mathtt{idx\_l} \leftarrow$ \texttt{argsort}($\set{\beta_{l}^{i}}_{i=1}^{K}$), $\mathtt{idx\_r} \leftarrow$ \texttt{argsort}($\set{\beta_{r}^{i}}_{i=1}^{K}$) \Comment{Ascending order}
        \State $H_{l}[i] := \frac{\left(\sum_{j=0}^{i}{\gamma^{\mathtt{idx\_l}[i]}x_{\mathcal{S}}^{\mathtt{idx\_l}[i]}}\right)^{2}}{\sum_{j=0}^{i}{\gamma^{\mathtt{idx\_l}[i]}(1-\gamma^{\mathtt{idx\_l}[i]})x_{\mathcal{S}}^{\mathtt{idx\_l}[i]}\left( 1-\frac{x_{\mathcal{S}}^{\mathtt{idx\_l}[i]}}{n^{\mathtt{idx\_l}[i]}}\right) }}$, $H_{l}^{max} \leftarrow \underset{i=0,\hdots,K-1}{\max}H_{l}[i]$
        \State $H_{r}[i] := \frac{\left(\sum_{j=0}^{i}{(1-\gamma^{\mathtt{idx\_r}[i]})x_{\mathcal{S}}^{\mathtt{idx\_r}[i]}}\right)^{2}}{\sum_{j=0}^{i}{\gamma^{\mathtt{idx\_r}[i]}(1-\gamma^{\mathtt{idx\_r}[i]})x_{\mathcal{S}}^{\mathtt{idx\_r}[i]}\left( 1-\frac{x_{\mathcal{S}}^{\mathtt{idx\_r}[i]}}{n^{\mathtt{idx\_r}[i]}}\right) }}$, $H_{r}^{max} \leftarrow \underset{i=0,\hdots,K-1}{\max}H_{r}[i]$
        \State Return $\max(H_{l}^{max},H_{r}^{max})$
	\EndFunction
 \end{algorithmic}
 \end{small}
 \end{algorithm}
 
Proposition~\ref{prop:pruning_definition} corresponds to the most general formulation of Tarone-based search space pruning in significant pattern mining. Nonetheless, it would only be useful if $T_{\mathrm{prune}}$ could be evaluated efficiently. A priori, the problem for the CMH test statistic seems intractable, since it corresponds to a discrete optimisation of a non-monotonic function over a set of size $O(\prod_{i=1}^{K}{\min(n_{1}^{i},n^{i}-n_{1}^{i})})$. However we have devised an $O(K \log(K))$ algorithm to evaluate $T_{\mathrm{prune}}$ for the CMH test statistic exactly:
\begin{theorem}
\label{theo:eval_max_corner}
	Algorithm~\ref{alg:eval_max_corner} computes the exact solution to the optimisation problem
	\begin{equation}
		\underset{\set{0 \le z^{i} \le x_{\mathcal{S}}^{i}}_{i=1}^{K}}{\max}T_{\mathrm{cmh}}^{\mathrm{max}}(\set{z^{i}}_{i=1}^{K})
	\end{equation}
		in the region $x_{\mathcal{S}}^{i} \in [0,\min(n_{1}^{i},n^{i}-n_{1}^{i})] \; \forall i=1,\hdots,K$ in $O(K \log(K))$ time.
\end{theorem}
The complexity statement in Theorem~\ref{theo:eval_max_corner} follows immediately from the pseudocode since, for large $K$, the runtime is dominated by the sorting operation. The proof of correctness can be found in the Appendix. 

Equipped with Algorithm~\ref{alg:eval_max_corner}, FastCMH prunes the search space efficiently and correctly in line~\ref{algline:evaluate_pruning} of the function \texttt{process\_next} with a call to the wrapper function \texttt{is\_not\_prunable}. The latter evaluates $T_{\mathrm{prune}}$ for the current pattern $\mathcal{S}$ by using Algorithm~\ref{alg:eval_max_corner} and return True if children of pattern $\mathcal{S}$ cannot be pruned for the search space. Whenever the search space space cannot be pruned, the block of code between lines~\ref{algline:cont_enum_begin} and \ref{algline:cont_enum_end} iteratively continues exploring the search space $\mathcal{M}$ using a depth-first strategy. The algorithm naturally stops when no more testable patterns remain at a certain level $\delta$, at which point the algorithm returns the corrected significance threshold $\delta^{*}_{\mathrm{tar}}= \alpha / \left\vert\mathcal{I}_{T}(\delta)\right\vert$, that is, the desired target FWER $\alpha$ over the final number of testable patterns $\left\vert\mathcal{I}_{T}(\delta)\right\vert$.

\section{Experiments} \label{sec:exp}

This section describes experiments which exhibit the three main contributions of the \texttt{FastCMH} algorithm: its improved detection performance (statistical power), its superior speed, and its ability to handle categorical data, when compared to approaches that use naive multiple testing correction procedures and/or do not take categorical information into account. 
The advantages of the \texttt{FastCMH} algorithm become clear in situations where the data can be split into several categories and there are a huge number of tests; then using Tarone's testability criterion can require many computations of the maximum possible value of the CMH test statistic.
We study the following instance of pattern mining: suppose we are given binary $x_i[j]$ sequences with class labels $y_i \in \{0,1 \}$, and the sequences may be split into categories $1,2, \dots, K$. We try to find subsequences $[\tau_m:\tau_m + \ell_m]$ such that $\max x_i[\tau_m:\tau_m + \ell_m]$, the maximum over all binary variables in this sequence, correlates with class membership.

\textbf{Simulation study} We compare three main methods: (1) \texttt{FastCMH}, our proposed method uses which Tarone's testability and efficiently computes the CMH statistic, (2) \texttt{FAIS-$\chi^2$}, which uses Tarone's testability, but does not take the categories into account, and so does not use the CMH test, (3) \texttt{Bonferroni-CMH}, which does not use Tarone's testability, but does use the CMH test. 
In order to investigate the difference in power between the three methods, we contruct $n$ binary sequences $x_i$ of length $L$, and with $K$ classes, there are $n/K$ samples in each class. In each class, half of the samples have label $y_i=0$ (controls) and the other half have label $y_i=1$ (cases). 
Initially each element $x_i[j]$ is sampled from a Bernoulli distribution with probability $p_1$ of being a $1$, i.e. $x_i[j] \sim \text{B}(1, p_1)$. A significant subsequence starting at $\tau_k$ with length $\ell_k$, $x_i[\tau_k:\tau_k+\ell_k-1]$ is created by sampling each element from a Bernoulli distribution with parameter $1-(1-p_{\text{case}})^{\ell_k}$ if $y_i=0$. This ensures that the probability of at least one 1 occurring in the subsequence $x_i[\tau_k:\tau_k+\ell_k-1]$ is $p_{\text{case}}$. If the label $y_i=0$, then the subsequence is left as before (with probability $p_1$ of each entry being a $1$).


\textit{Power} There are two complementary situations where \texttt{FastCMH} has improved power: first, it has improved detection power of true significant subsequences, due its use of Tarone's testability criterion, when compared to \texttt{Bonferroni-CMH}; second, it will often (correctly) omit subsequences which appear to be significant, but are actually highly correlated with the categorical labels rather than the class labels. 
Recall that if $\beta$ is the probability of a method returning a false negative (i.e. the Type II error), then the (statistical) power is defined as $1-\beta$. In other words, when attempting to detect significant subsequences, it is the probability of correctly detecting a significant subsequence. 
Figure~\ref{fig:exp:main}(a) shows the results of running the \texttt{FastCMH}, \texttt{FAIS-$\chi^2$} and \texttt{Bonferroni-CMH} methods over the data generated with parameters: $n=500$, $L=10,000$, $p_1=0.2$, $K=2$, and one significant subsequence of length $5$ at location $2500$ in each sequence. In this setting, each category is generated in the same manner. 

One notices that both \texttt{FastCMH} and \texttt{FAIS-$\chi^2$} have higher power than \texttt{Bonferroni-CMH} for $p_{\text{case}} \in [0.3, 0.8]$. 
In fact the results for \texttt{FastCMH} and \texttt{FAIS-$\chi^2$} are so similar that they are virtually identical.
A similar experiment can be performed, by fixing the value of $p_{\text{case}}$ and allowing the length of $L$ to vary between $1000$ and $100,000$, while keeping the other parameters the same as before. The results of this experiment are shown in the Supplmentary material, and show that when $L$ is very large, the power of the \texttt{Bonferroni-CMH} method can become zero, while the power for \texttt{FastCMH} and \texttt{FAIS-$\chi^2$} are non-zero. 

\begin{figure}
\centering
\begin{subfigure}{0.32\textwidth}
  \centering
  \includegraphics[width=\linewidth]{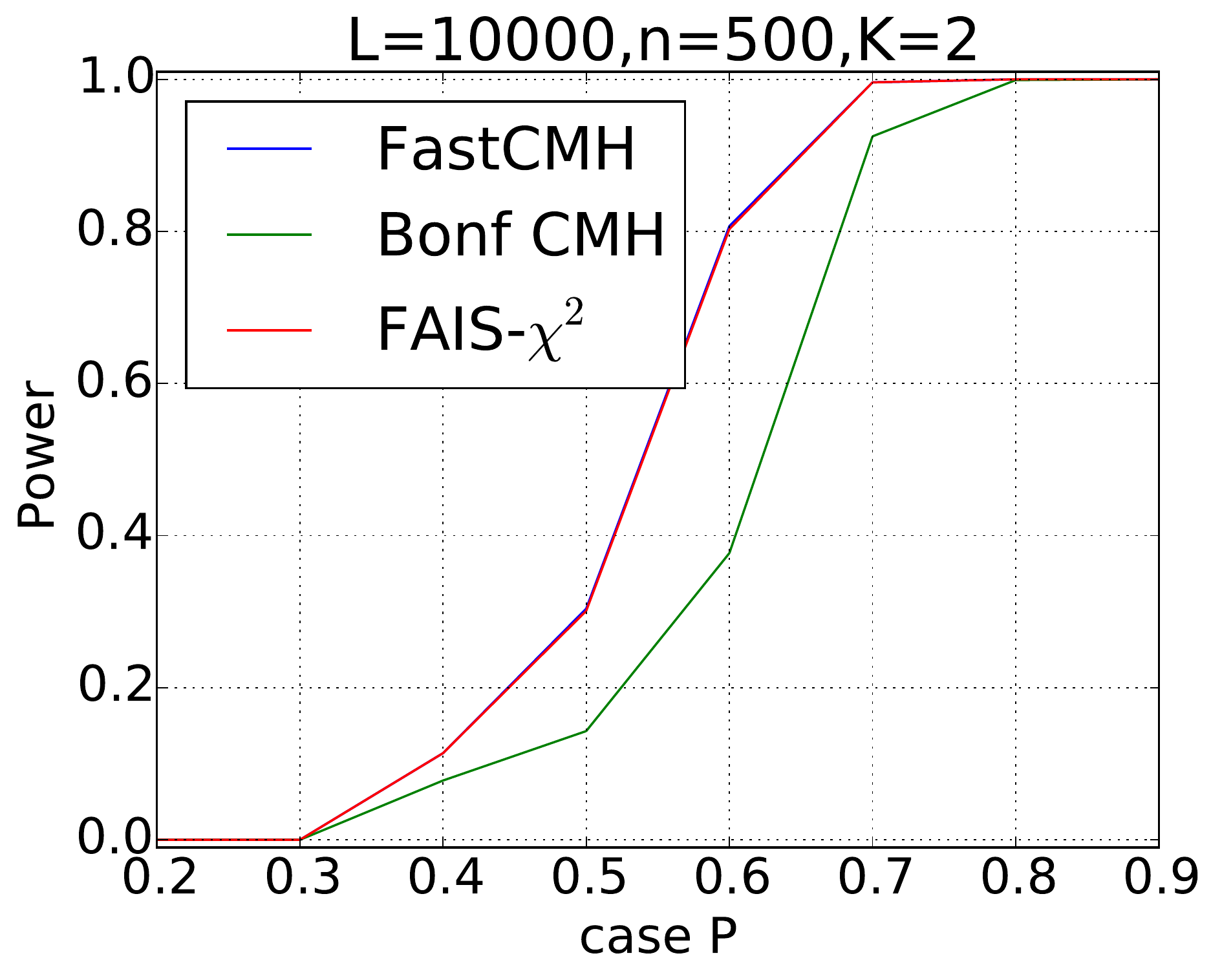}
  \caption{ }
  \label{fig:sub1power}
\end{subfigure}%
\begin{subfigure}{0.32\textwidth}
  \centering
  \includegraphics[width=\linewidth]{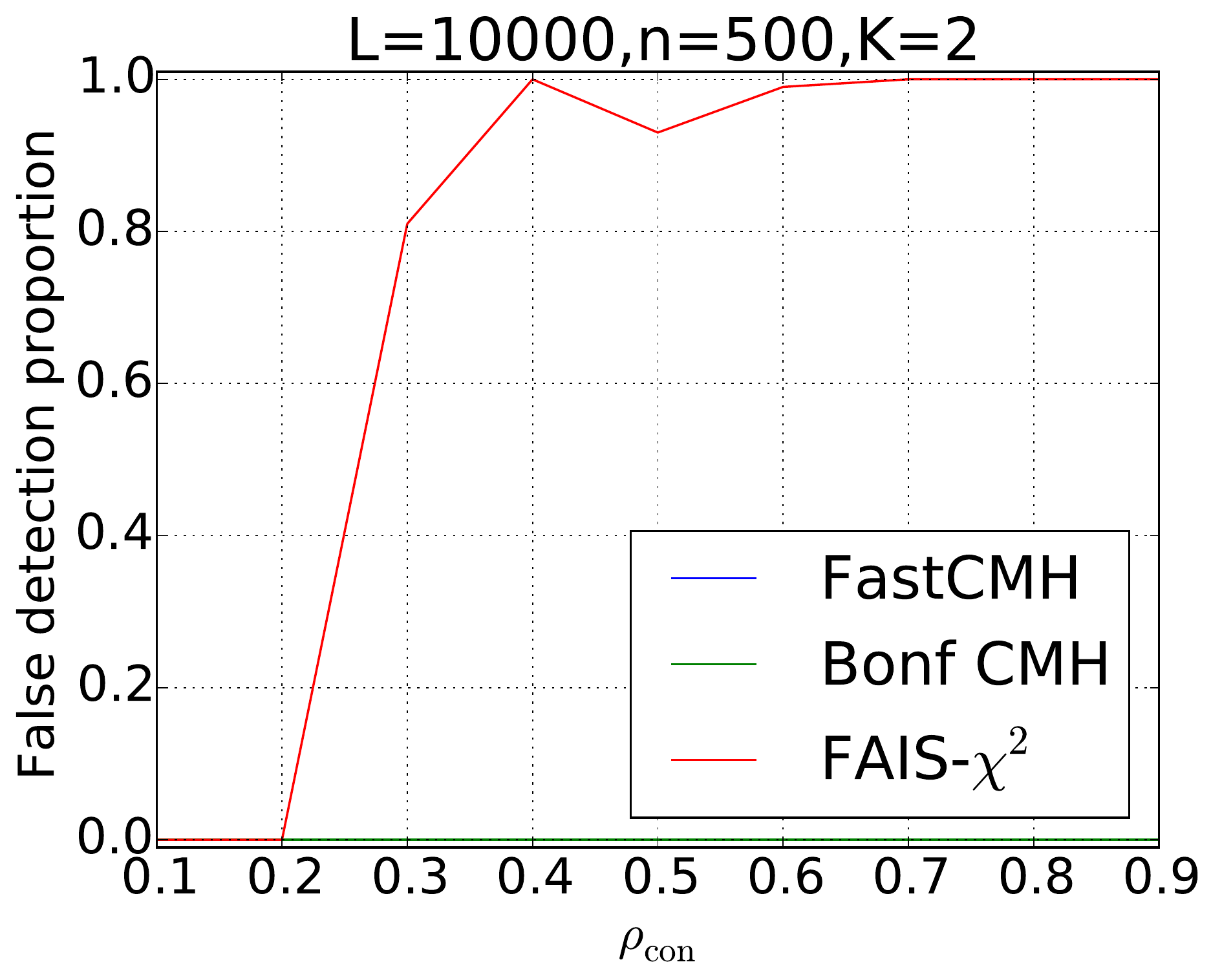}
  \caption{ }
  \label{fig:sub2power}
\end{subfigure}
\begin{subfigure}{0.32\textwidth}
  \centering
  \includegraphics[width=\linewidth]{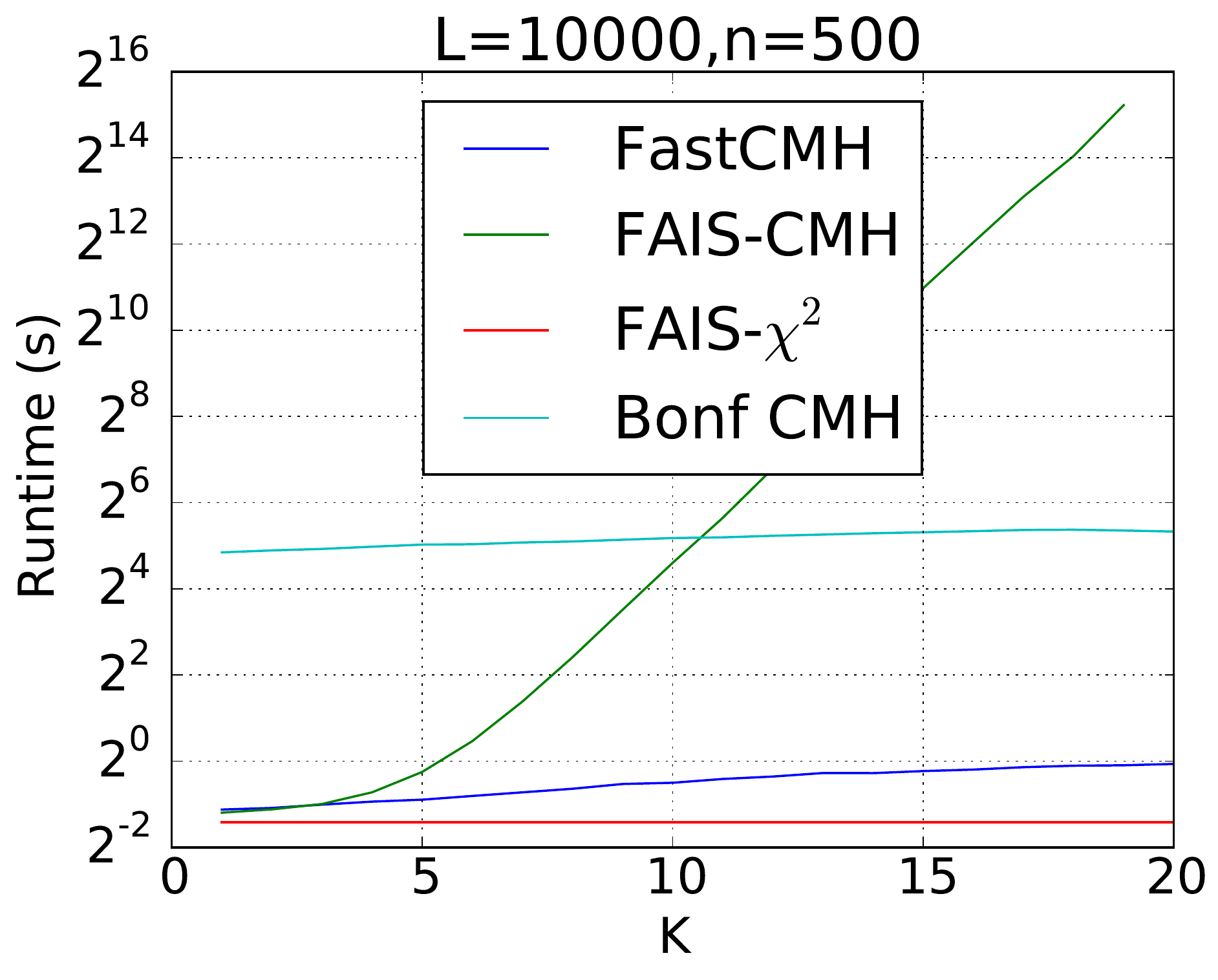}
  \caption{ }
  \label{fig:sub3speed}
\end{subfigure}
\caption{(a) A comparison of the power of \texttt{FastCMH}, \texttt{FAIS-$\chi^2$} and \texttt{Bonferroni-CMH} for detecting true significant subsequences, as $p_{\text{case}}$ varies. (b) The proportion of confounded significant subsequences falsely detected by each of those three algorithms. (c) A plot showing the runtimes of \texttt{FastCMH}, \texttt{FAIS-$\chi^2$}, \texttt{FAIS-CMH} and \texttt{Bonferroni CMH}. }
\label{fig:exp:main}
\end{figure}
So far, the sequences for the different categories have been generated in the same manner, which leads to little difference between \texttt{FastCMH} and \texttt{FAIS-$\chi^2$}. In the Supplementary material a procedure is described for contstructing a \emph{confounded significant subsequence}, which is a subsequence that is highly correlated with the categorical labels, and the categorical labels themselves are correlated with the class labels $y_i$. An experiment, the same as before but now with only these confounded significant subsequences, can be conducted. The results, in Figure~\ref{fig:exp:main}(b), show that \texttt{FastCMH} and \texttt{Bonferroni-CMH} do not detect these confounded significant subsequences, but the \texttt{FAIS-$\chi^2$} method, which does not take categorical labels into account and does not use the CMH test, often incorrectly detects these confounded significant subsequences. This illustrates the value of using the CMH test, and in particular the \texttt{FastCMH} method. In the next set of experiments, we show that \texttt{FastCMH} is significantly faster and more efficient than the brute-force \texttt{Bonferroni-CMH} method.

\textit{Speed} In addition to the three methods described above, we investigate the speed of (4) \texttt{FAIS-CMH}, a method which uses Tarone's testability criterion, but computes CMH in a naive way. Theorem~1 shows that \texttt{FastCMH} is equivalent to this naive approach, so it was not included in the power simulations above. 
Figure~\ref{fig:exp:main}(c) shows that \texttt{FastCMH} is several orders of magnitudes faster than \text{Bonferroni-CMH} and \texttt{FAIS-CMH}. In fact, \texttt{FAIS-CMH}, which computes the maximum CMH statistic in a naive manner, increases drastically as $K$ increases, while \texttt{FastCMH} increases almost linearly. \texttt{FastCMH} is comparable in speed to \texttt{FAIS-$\chi^2$}, but recall that \texttt{FAIS-$\chi^2$} does not take the categories $K$ into account, and that Figure~\ref{fig:exp:main}(b) shows that it is susceptible to falsely detecting confounded significant subsequences. Overall, these experiments have shown that our method \texttt{FastCMH} is significantly faster and has better detection performance in comparison to competitor methods.


\textbf{Real Data Application: Arabidopsis}
We perform significant subsequence search on genetic data from \textit{Arabidopsis thaliana}~\cite{atwell2010genome}, a plant model organism. Each plant is represented by a sequence of 214,051 binary variables (variants in its DNA sequence) and belongs to one of two classes (phenotypes). Significant patterns are subsequences of binary sequence variants 
that are correlated with the phenotype of the plant. We focus on three relevant datasets, each describing the presence/absence of a particular phenotype (YEL, LY, LES). Covariates represent geographic origin of the plants, here represented by four distinct clusters. We compare FastCMH to a $\chi^2$-test on each dataset and observe that conditioning on the covariates drastically reduces the number of significant subsequences found. Randomly permuting the covariates and repeating the experiment 150 times shows that this is a statistically significant reduction in the number of patterns found. The p-values based on this null distribution associated to each dataset are (0.007,0.007,0.070), respectively. This indicates that FastCMH manages to remove false positives findings that were caused by not considering geographic origin as a covariate. (See Supplement for a detailed description of the experiment.)

 \begin{table}[t]
 \centering 
\caption{Number of significant patterns found by $\chi^2$-test and FastCMH}
\begin{tabular}{rlll}
 \hline
Datasets & YEL & LY & LES \\
 \hline
 $\chi^2$-test & 1069 & 26 & 20 \\
FastCMH & 73 & 2 & 8 \\
  \hline
\end{tabular}
\label{tb:res}
\end{table}

\section{Conclusions}

In this article, we have presented an approach to significant pattern mining in the presence of categorical covariates with $K$ states. Key to our approach is that we define a variant of the Cochran-Mantel-Haenszel test for $K$ 2 $\times$ 2 contingency tables~\cite{mantel1959}, which prunes the vast search space of candidate patterns by excluding {\it non-testable} patterns.
We propose an efficient, exact algorithm to compute this test. Its runtime is only linear in $K$, whereas the standard approach scales exponentially in $K$.
In our experiments, it combines computational efficiency with high statistical power on simulated and real-world datasets. Our findings open the door to important applications of significant pattern mining, for instance, in personalized medicine, which we will explore in future work.

\clearpage
\appendix

\section{Appendix}

This appendix provides proofs of the theoretical results described in the main part of the paper, further details about the experimental study, as well as additional experiments.

\section*{Proofs}

\begin{proposition}
	\begin{equation}
	T_{\mathrm{cmh}}^{\mathrm{max}}(\set{x_{\mathcal{S}}^{i}}_{i=1}^{K}) = \frac{\max\left[ \left( \sum_{i=1}^{K}{a_{\mathcal{S},min}^{i} - \gamma^{i}x_{\mathcal{S}}^{i}} \right)^{2}   , \left( \sum_{i=1}^{K}{a_{\mathcal{S},max}^{i} - \gamma^{i}x_{\mathcal{S}}^{i}}\right)^{2} \right]}{\sum_{i=1}^{K}{\gamma^{i}(1-\gamma^{i})x_{\mathcal{S}}^{i}\left( 1-\frac{x_{\mathcal{S}}^{i}}{n^{i}}\right) }}
	\end{equation}
\end{proposition}
\begin{proof}
	The CMH test statistic can be written as:
    \begin{equation}
    	T_{\mathrm{cmh}} = \frac{\left(\sum_{i=1}^{K}{a_{\mathcal{S}}^{i} - \gamma^{i}x_{\mathcal{S}}^{i}}\right)^{2}}{\sum_{i=1}^{K}{\gamma^{i}(1-\gamma^{i})x_{\mathcal{S}}^{i}\left( 1-\frac{x_{\mathcal{S}}^{i}}{n^{i}}\right) }} = \frac{\left(a_{\mathcal{S}}-\sum_{i=1}^{K}{ \gamma^{i}x_{\mathcal{S}}^{i}}\right)^{2}}{\sum_{i=1}^{K}{\gamma^{i}(1-\gamma^{i})x_{\mathcal{S}}^{i}\left( 1-\frac{x_{\mathcal{S}}^{i}}{n^{i}}\right) }}
    \end{equation}
     where $a_{\mathcal{S}} = \sum_{i=1}^{K}{a_{\mathcal{S}}^{i}}$. Moreover, we have $a_{\mathcal{S}} \in \left[\sum_{i=1}^{K}{a_{\mathcal{S},min}^{i}},\sum_{i=1}^{K}{a_{\mathcal{S},max}^{i}}\right]$. Clearly, $T_{\mathrm{cmh}}$ is a quadratic function of $a_{\mathcal{S}}$ with positive semidefinite Hessian over a compact set. Therefore, its maximum must be attained in the extremes; that is, either when $a_{S}^{i}=a_{\mathcal{S},min}^{i} \; \forall i=1,\hdots,K$ or when $a_{S}^{i}=a_{\mathcal{S},max}^{i} \; \forall i=1,\hdots,K$.
\end{proof}
 \begin{proposition}
	The minimum attainable $p$-value function for the CMH statistic, $\Psi(\set{x_{\mathcal{S}}^{i}}_{i=1}^{K})$, is not monotonically decreasing with respect to each $x_{\mathcal{S}}^{i}$ in $x_{\mathcal{S}}^{i} \in [0,\min(n_{1}^{i},n^{i}-n_{1}^{i})]$ if there exist $i$,$j$ in $1,\hdots,K$ such that $\frac{n_{1}^{i}}{n^{i}} < \frac{1}{2},\frac{n_{1}^{j}}{n^{j}} > \frac{1}{2}$
\end{proposition}
\begin{proof}
	We just need to show that the case
	\begin{equation}
		\underset{\set{0 \le z^{i} \le x_{\mathcal{S}}^{i}}_{i=1}^{K}}{\max}T_{\mathrm{cmh}}^{\mathrm{max}}(\set{z^{i}}_{i=1}^{K}) \neq (x_{\mathcal{S}}^{1},x_{\mathcal{S}}^{2},\hdots,x_{\mathcal{S}}^{K})
	\end{equation}
	can occur in practice. That can be readily seen during the proof of Theorem 1, so we defer the proof of Proposition 2 to that.
\end{proof}

\begin{theorem}
	Algorithm 2 computes the exact solution to the optimisation problem
	\begin{equation}
		\underset{\set{0 \le z^{i} \le x_{\mathcal{S}}^{i}}_{i=1}^{K}}{\max}T_{\mathrm{cmh}}^{\mathrm{max}}(\set{z^{i}}_{i=1}^{K})
	\end{equation}
	in the region $x_{\mathcal{S}}^{i} \in [0,\min(n_{1}^{i},n^{i}-n_{1}^{i})] \; \forall i=1,\hdots,K$ in $O(K \log(K))$ time.
\end{theorem}

We decompose the proof of Theorem~\ref{theo:eval_max_corner} in several parts. First of all, note that when $x_{\mathcal{S}}^{i} \in [0,\min(n_{1}^{i},n^{i}-n_{1}^{i})] \; \forall i=1,\hdots,K$, the maximum attainable CMH test statistic $T_{\mathrm{cmh}}^{\mathrm{max}}$ can be written as  $T_{\mathrm{cmh}}^{\mathrm{max}}(\set{x_{\mathcal{S}}^{i}}_{i=1}^{K}) = \max(T_{l}(\set{x_{\mathcal{S}}^{i}}_{i=1}^{K}),T_{r}(\set{x_{\mathcal{S}}^{i}}_{i=1}^{K}))$ where:

	\begin{align}
		T_{l}(\set{x_{\mathcal{S}}^{i}}_{i=1}^{K}) &= \frac{\left(\sum_{i=1}^{K}{\gamma^{i}x_{\mathcal{S}}^{i}} \right)^{2}}{\sum_{i=1}^{K}{\gamma^{i}(1-\gamma^{i})x_{\mathcal{S}}^{i}\left( 1-\frac{x_{\mathcal{S}}^{i}}{n^{i}}\right) }} \\
		T_{r}(\set{x_{\mathcal{S}}^{i}}_{i=1}^{K}) &= \frac{\left( \sum_{i=1}^{K}{(1-\gamma^{i})x_{\mathcal{S}}^{i}}\right)^{2} }{\sum_{i=1}^{K}{\gamma^{i}(1-\gamma^{i})x_{\mathcal{S}}^{i}\left( 1-\frac{x_{\mathcal{S}}^{i}}{n^{i}}\right) }}
	\end{align}
	
Using that, first we prove the following lemmas:
\begin{lemma}
\label{lemma:lemma1}
	The functions $T_{l}(\set{z^{i}}_{i=1}^{K}), T_{r}(\set{z^{i}}_{i=1}^{K})$ in the region $0 \le z^{i} \le x_{\mathcal{S}}^{i}$ with $x_{\mathcal{S}}^{i} \le \min(n_{1}^{i},n^{i}-n_{1}^{i})] \; \forall \, i=1,\hdots,K$ as a function of a single variable $z^{j}$ while keeping the other $K-1$ variables fixed satisfies one of the three conditions: (1) it is monotonically increasing in $0 \le z^{j} \le x_{\mathcal{S}}^{j}$; (2) it is monotonically decreasing in $0 \le z^{j} \le x_{\mathcal{S}}^{j}$ or (3) has a single local minimum in $0 \le z^{j} \le x_{\mathcal{S}}^{j}$.
    
    As a consequence, both functions $T_{l}(\set{z^{i}}_{i=1}^{K}), T_{r}(\set{z^{i}}_{i=1}^{K})$ are maximized with respect to $z^{j}$ while the other variables are kept fixed either when $z^{j}=0$ or when $z^{j}=x_{\mathcal{S}}^{j}$.
\end{lemma}
\begin{proof}
	Computing the partial derivative of $T_{l}(\set{z^{i}}_{i=1}^{K})$ and $T_{r}(\set{z^{i}}_{i=1}^{K})$ with respect to $z^{j}$ yields:
    \begin{equation*}
    	\frac{\partial T_{l}(\set{z^{i}}_{i=1}^{K})}{\partial z^{j}} = \Lambda_{l}(\set{z^{i}}_{i=1}^{K}))A_{l}(\set{z^{i}}_{i=1}^{K}))
    \end{equation*}
    \begin{equation*}
    	\frac{\partial T_{r}(\set{z^{i}}_{i=1}^{K})}{\partial z^{j}} = \Lambda_{r}(\set{z^{i}}_{i=1}^{K}))A_{r}(\set{z^{i}}_{i=1}^{K}))
    \end{equation*}
    with
    \begin{equation*}
    	\Lambda_{l}(\set{z^{i}}_{i=1}^{K}) = \frac{\gamma^{j}\sum_{i=1}^{K}{\gamma^{i}z^{i}}}{\left(\sum_{i=1}^{K}{\gamma^{i}(1-\gamma^{i})z^{i}(1-\frac{z^{i}}{n^{i}})}\right)^{2}}
    \end{equation*}
    \begin{equation*}
    	\Lambda_{r}(\set{z^{i}}_{i=1}^{K})= \frac{(1-\gamma^{j})\sum_{i=1}^{K}{(1-\gamma^{i})z^{i}}}{\left(\sum_{i=1}^{K}{\gamma^{i}(1-\gamma^{i})z^{i}(1-\frac{z^{i}}{n^{i}})}\right)^{2}}
    \end{equation*}
    \begin{equation*}
    	A_{l}(\set{z^{i}}_{i=1}^{K}) = \sum_{i=1,i\not=j}^{K}{\gamma^{i}z^{i}\left(2(1-\gamma^{i})(1-\frac{z^{i}}{n^{i}})-(1-\gamma^{j})\right)} + (1-\gamma^{j})\left(\gamma^{j}+\frac{2}{n^{j}}\sum_{i=1,i\not=j}^{K}{\gamma^{i}z^{i}}\right)z^{j}
    \end{equation*}
    \begin{equation*}
    	A_{r}(\set{z^{i}}_{i=1}^{K}) = \sum_{i=1,i\not=j}^{K}{(1-\gamma^{i})z^{i}\left(2\gamma^{i}(1-\frac{z^{i}}{n^{i}})-\gamma^{j}\right)} + \gamma^{j}\left((1-\gamma^{j})+\frac{2}{n^{j}}\sum_{i=1,i\not=j}^{K}{(1-\gamma^{i})z^{i}}\right)z^{j}
    \end{equation*}
    Because $\Lambda_{l}(\set{z^{i}}_{i=1}^{K}) \ge 0$ and $\Lambda_{r}(\set{z^{i}}_{i=1}^{K}) \ge 0$, the sign of the partial derivatives are determined by the sign of $A_{l}(\set{z^{i}}_{i=1}^{K}))$ and $A_{r}(\set{z^{i}}_{i=1}^{K}))$ respectively. In both cases, $A(\set{z^{i}}_{i=1}^{K}))$ can be expressed as $A(\set{z^{i}}_{i=1}^{K}))=b(\set{z^{i}}_{i=1,i \neq j}^{K}) + \mu(\set{z^{i}}_{i=1,i \neq j}^{K})z^{j}$. That is, as an affine function of $z^{j}$ where the intersect and slope is controlled by all other $K-1$ variables. Moreover, regardless of $\set{z^{i}}_{i=1,i \neq j}^{K}$, $\mu(\set{z^{i}}_{i=1,i \neq j}^{K}) \ge 0$. Therefore the partial derivatives are either always positive, always negative, or negative until a unique point where it crosses zero (minimum) and then positive. 
\end{proof}

\begin{lemma}
\label{lemma:vertices}
	The functions $T_{l}(\set{z^{i}}_{i=1}^{K})$ and $T_{r}(\set{z^{i}}_{i=1}^{K})$ in the region $0 \le z^{i} \le x_{\mathcal{S}}^{i}$, with $x_{\mathcal{S}}^{i} \le \min(n_{1}^{i},n^{i}-n_{1}^{i})] \; \forall \, i=1,\hdots,K$, are maximized when either $z^{i}=0$ or $z^{i}=x_{\mathcal{S}}^{i}$ for all $i=1,\hdots,K$. In other words, the size of the potential set of optimal values for $\set{z^{i}}_{i=1}^{K}$ is $2^{K}$.
\end{lemma}

\begin{proof}
	The proof is analogous for both $T_{l}(\set{z^{i}}_{i=1}^{K})$ and $T_{r}(\set{z^{i}}_{i=1}^{K})$. Therefore, we focus on $T_{l}(\set{z^{i}}_{i=1}^{K})$ without loss of generality.
	We can write:
	\begin{equation}
		\underset{\set{0 \le z^{i} \le x_{\mathcal{S}}^{i}}_{i=1}^{K}}{\max}T_{l}(\set{z^{i}}_{i=1}^{K}) = \underset{z^{1}}{\max}\,\underset{z^{2}}{\max}\hdots\underset{z^{K}}{\max}\,T_{l}(\set{z^{i}}_{i=1}^{K})
	\end{equation}
	By Lemma~\ref{lemma:lemma1}:
		\begin{equation}
	 		\underset{z^{K}}{\max}\,T_{l}(\set{z^{i}}_{i=1}^{K}) = \max \left(T_{l}(\set{z^{i}}_{i=1,i\neq K}^{K},0),T_{l}(\set{z^{i}}_{i=1,i\neq K}^{K},x_{\mathcal{S}}^{K}) \right)
		\end{equation}
	 The same argument can be then applied recursively to the functions $T_{l}(\set{z^{i}}_{i=1,i\neq K}^{K},0)$ and $T_{l}(\set{z^{i}}_{i=1,i\neq K}^{K},x_{\mathcal{S}}^{K})$ in the variable $z^{K-1}$ and so on until the conclusion of the theorem follows.
\end{proof}

One can rewrite both function $T_{l}(\set{z^{i}}_{i=1}^{K})$ and $T_{r}(\set{z^{i}}_{i=1}^{K})$ generically as:
\begin{align}
	T = \frac{\left(\sum_{i=1}^{K}{l_{i}(\alpha_{i})}\right)^{2}}{\sum_{i=1}^{K}{\alpha_{i}l_{i}(\alpha_{i})}}
\end{align}
with $\alpha_{i} \in [0,1]$ and $l_{i}(\alpha_{i})>0$. In particular, $\alpha_{i}=(1-\gamma^{i})(1-\frac{x_{\mathcal{S}}^{i}}{n^{i}})$, $l_{i}(\alpha_{i})=n_{1}^{i}\left(1-\frac{\alpha_{i}}{1-\gamma^{i}}\right)$ for $T_{l}(\set{z^{i}}_{i=1}^{K})$ and $\alpha_{i}=\gamma^{i})(1-\frac{x_{\mathcal{S}}^{i}}{n^{i}})$, $l_{i}(\alpha_{i})=(n^{i}-n_{1}^{i})\left(1-\frac{\alpha_{i}}{1-\gamma^{i}}\right)$ for $T_{r}(\set{z^{i}}_{i=1}^{K})$. Let us introduce $K$ binary indicator variables $\delta_{1},\hdots,\delta_{K}$ in $T$:
\begin{equation}
	T(\delta_{1},\hdots,\delta_{K}) = \frac{\left(\sum_{i=1}^{K}{\delta_{i}l_{i}(\alpha_{i})}\right)^{2}}{\sum_{i=1}^{K}{\delta_{i}\alpha_{i}l_{i}(\alpha_{i})}}
\end{equation}
The only actual requirements we assume for $T(\delta_{1},\hdots,\delta_{K})$ as defined above are that with $\alpha_{i} \in [0,1]$ and $l_{i}(\alpha_{i})>0$. Suppose further that:
\begin{equation}
\label{eq:theo:condition}
	\underset{\delta_{1},\hdots,\delta_{K}}{\argmax} \; T(\delta_{1},\hdots,\delta_{K}) = (\underbrace{1,1,\hdots,1}_{R},\underbrace{0,0,\hdots,0}_{K-R})
\end{equation}
holds with $R>0$. Informally, equation~\eqref{eq:theo:condition} means that the maximum is achieved by keeping the terms in the summation corresponding to the $R$ smallest $\alpha_{i}$. Since we know that $T_{\mathrm{cmh}}^{\mathrm{max}}(\set{x_{\mathcal{S}}^{i}}_{i=1}^{K}) = \max(T_{l}(\set{x_{\mathcal{S}}^{i}}_{i=1}^{K}),T_{r}(\set{x_{\mathcal{S}}^{i}}_{i=1}^{K}))$, and both $T_{l}(\set{z^{i}}_{i=1}^{K})$ and $T_{r}(\set{z^{i}}_{i=1}^{K})$ are maximised when either $z^{i}=0$ or $z^{i}=x_{\mathcal{S}}^{i}$ by lemma~\ref{lemma:vertices}, then it follows that if $\alpha_{j} > \alpha_{i}$ and $z^{j,*}=x_{\mathcal{S}}^{j}$ then $z^{i,*}=x_{\mathcal{S}}^{i}$. But that is exactly the condition that Algorithm 2 exploits to reduce the computational complexity of evaluating $T_{\mathrm{prune}}$ from the value $O(2^{K})$ resulting from lemma~\ref{lemma:vertices} to just $O(K\log(K)$. In other words, to finish proving Theorem 1, we just need to show that equation~\eqref{eq:theo:condition} holds.

We will prove it by induction. First, we show that the statement holds for $K=2$.

\begin{lemma}
	Following all the previous definitions, we have that:
	\begin{equation}
		\underset{\delta_{1},\delta_{2}}{\argmax} \; T(\delta_{1},\delta_{2}) \in \left\{(1,0),(1,1)\right\}
	\end{equation}
\end{lemma}
\begin{proof}
	The only possible contradicting case would be $\underset{\delta_{1},\delta_{2}}{\argmax} \; T(\delta_{1},\delta_{2}) = (0,1)$, since the case $(0,0)$ yields a trivial value for the function T. We show directly that under the assumption $\alpha_{1} \le \alpha_{2}$, the contradiction cannot happen. Indeed we have:
	
	\begin{align}
	\frac{(l_{1}(\alpha_{1})+l_{2}(\alpha_{2}))^{2}}{\alpha_{1}l_{1}(\alpha_{1})+\alpha_{2}l_{2}(\alpha_{2})} - \frac{l^{2}_{2}(\alpha_{2})}{\alpha_{2}l_{2}(\alpha_{2})} &= \frac{(l_{1}(\alpha_{1})+l_{2}(\alpha_{2}))^{2}\alpha_{2}l_{2}(\alpha_{2}) -  l^{2}_{2}(\alpha_{2})(\alpha_{1}l_{1}(\alpha_{1})+\alpha_{2}l_{2}(\alpha_{2}))}{(\alpha_{1}l_{1}(\alpha_{1})+\alpha_{2}l_{2}(\alpha_{2}))\alpha_{2}l_{2}(\alpha_{2})} \\ \nonumber
	 &= l_{1}(\alpha_{1})\frac{(l_{1}(\alpha_{1})+2l_{2}(\alpha_{2}))\alpha_{2}l_{2}(\alpha_{2}) - \alpha_{1}l^{2}_{2}(\alpha_{2})}{(\alpha_{1}l_{1}(\alpha_{1})+\alpha_{2}l_{2}(\alpha_{2}))\alpha_{2}l_{2}(\alpha_{2})} \\ \nonumber
	 &= l_{1}(\alpha_{1})\frac{\alpha_{2}l_{1}(\alpha_{1})l_{2}(\alpha_{2}) + (2\alpha_{2}-\alpha_{1})l^{2}_{2}(\alpha_{2})}{(\alpha_{1}l_{1}(\alpha_{1})+\alpha_{2}l_{2}(\alpha_{2}))\alpha_{2}l_{2}(\alpha_{2})}
	\end{align}
	
	Since $l_{i}(\alpha_{i}) \ge 0$ and $\alpha_{1} \le \alpha_{2}$, it follows that the numerator in the expression above is positive, thus $T(1,1)>T(0,1)$ contradicting the statement that $\underset{\delta_{1},\delta_{2}}{\argmax} \; T(\delta_{1},\delta_{2}) = (0,1)$.
\end{proof}

Now we prove the induction step. Suppose the statement holds for an arbitrary dimension $K$, we will show then it also holds for dimension $K+1$.

\begin{lemma}
	If we have:
	\begin{equation}
		\underset{\delta_{1},\hdots,\delta_{K}}{\argmax} \; \frac{\left(\sum_{i=1}^{K}{\delta_{i}l_{i}(\alpha_{i})}\right)^{2}}{\sum_{i=1}^{K}{\delta_{i}\alpha_{i}l_{i}(\alpha_{i})}} = (\underbrace{1,1,\hdots,1}_{R},\underbrace{0,0,\hdots,0}_{K-R})
\end{equation}
	Then:
	\begin{equation}
		\underset{\delta_{1},\hdots,\delta_{K},\delta_{K+1}}{\argmax} \; \frac{\left(\sum_{i=1}^{K}{\delta_{i}l_{i}(\alpha_{i})} + \delta_{K+1}l_{K+1}(\alpha_{K+1})\right)^{2}}{\sum_{i=1}^{K}{\delta_{i}\alpha_{i}l_{i}(\alpha_{i})} + \delta_{K+1}\alpha_{K+1}l_{K+1}(\alpha_{K+1})} = (\underbrace{1,1,\hdots,1}_{R^{\prime}},\underbrace{0,0,\hdots,0}_{(K+1)-R^{\prime}})
	\end{equation}
\end{lemma}
\begin{proof}
	We can start by writing:
	\begin{align}
	\underset{\delta_{1},\hdots,\delta_{K},\delta_{K+1}}{\max} \;  &\frac{\left(\sum_{i=1}^{K}{\delta_{i}l_{i}(\alpha_{i})} + \delta_{K+1}l_{K+1}(\alpha_{K+1})\right)^{2}}{\sum_{i=1}^{K}{\delta_{i}\alpha_{i}l_{i}(\alpha_{i})} + \delta_{K+1}\alpha_{K+1}l_{K+1}(\alpha_{K+1})} \\ \nonumber
	&= \max\left(\underset{\delta_{1},\hdots,\delta_{K}}{\max} \; \frac{\left(\sum_{i=1}^{K}{\delta_{i}l_{i}(\alpha_{i})}\right)^{2}}{\sum_{i=1}^{K}{\delta_{i}\alpha_{i}l_{i}(\alpha_{i})}},\underset{\delta_{1},\hdots,\delta_{K}}{\max} \; \frac{\left(\sum_{i=1}^{K}{\delta_{i}l_{i}(\alpha_{i})} + l_{K+1}(\alpha_{K+1}) \right)^{2}}{\sum_{i=1}^{K}{\delta_{i}\alpha_{i}l_{i}(\alpha_{i}) + \alpha_{K+1}l_{K+1}(\alpha_{K+1})}} \right)
	\end{align}
	
	If:
	\begin{equation}
	\label{eq:8}
		\underset{\delta_{1},\hdots,\delta_{K}}{\max} \; \frac{\left(\sum_{i=1}^{K}{\delta_{i}l_{i}(\alpha_{i})}\right)^{2}}{\sum_{i=1}^{K}{\delta_{i}\alpha_{i}l_{i}(\alpha_{i})}} \ge \underset{\delta_{1},\hdots,\delta_{K}}{\max} \; \frac{\left(\sum_{i=1}^{K}{\delta_{i}l_{i}(\alpha_{i})} + l_{K+1}(\alpha_{K+1}) \right)^{2}}{\sum_{i=1}^{K}{\delta_{i}\alpha_{i}l_{i}(\alpha_{i}) + \alpha_{K+1}l_{K+1}(\alpha_{K+1})}}
	\end{equation}
	Then the statement is trivially true. Suppose now that equation~\eqref{eq:8} does \textbf{not} hold. We show next that:
	\begin{equation}
	\label{eq:9}
		(\hat{\delta}_{1},\hdots,\hat{\delta}_{K}) = \underset{\delta_{1},\hdots,\delta_{K}}{\argmax} \; \frac{\left(\sum_{i=1}^{K}{\delta_{i}l_{i}(\alpha_{i})} + l_{K+1}(\alpha_{K+1}) \right)^{2}}{\sum_{i=1}^{K}{\delta_{i}\alpha_{i}l_{i}(\alpha_{i}) + \alpha_{K+1}l_{K+1}(\alpha_{K+1})}} = (\underbrace{1,1,\hdots,1}_{K})
	\end{equation}
	which would complete the proof. To show that equation~\eqref{eq:9} is true when equation~\eqref{eq:8} does \textbf{not} hold, we proceed by contradiction in two steps. First we prove that there is at most a single $j \in \left\{1,\hdots,K\right\} \; | \; \hat{\delta}_{j}=0$. To see that, suppose $\exists j \; | \; \hat{\delta}_{j}=0$ and define:
	\begin{equation}
		\widetilde{T}(\delta_{1},\hdots,\delta_{j-1},\delta_{j+1},\hdots,\delta_{K},\delta_{K+1}) =  \frac{\left(\sum_{i=1,i \neq j}^{K}{\delta_{i}l_{i}(\alpha_{i})} + \delta_{K+1}l_{K+1}(\alpha_{K+1}) \right)^{2}}{\sum_{i=1,i \neq j}^{K}{\delta_{i}\alpha_{i}l_{i}(\alpha_{i}) + \delta_{K+1}\alpha_{K+1}l_{K+1}(\alpha_{K+1})}}
	\end{equation} 
	which is nothing but $T(\delta_{1},\delta_{2},\hdots,\delta_{K},\delta_{K+1})$ with the $j$-th term removed. Note that, since $\hat{\delta}_{j}=0$, we have:
	\begin{equation}
		 \underset{\delta_{1},\hdots,\delta_{j-1},\delta_{j+1},\hdots,\delta_{K},\delta_{K+1}}{\max} \; \widetilde{T}(\delta_{1},\hdots,\delta_{j-1},\delta_{j+1},\hdots,\delta_{K},\delta_{K+1}) \ge  \underset{\delta_{1},\hdots,\delta_{K}}{\max} \; \frac{\left(\sum_{i=1}^{K}{\delta_{i}l_{i}(\alpha_{i})} + l_{K+1}(\alpha_{K+1}) \right)^{2}}{\sum_{i=1}^{K}{\delta_{i}\alpha_{i}l_{i}(\alpha_{i}) + \alpha_{K+1}l_{K+1}(\alpha_{K+1})}}
	\end{equation}
	Moreover, since equation~\eqref{eq:8} does not hold, it follows that:
	\begin{equation}
		(\tilde{\delta}_{1},\hdots,\tilde{\delta}_{j-1},\tilde{\delta}_{j+1},\hdots,\tilde{\delta}_{K},\tilde{\delta}_{K+1}) = \underset{\delta_{1},\hdots,\delta_{j-1},\delta_{j+1},\hdots,\delta_{K},\delta_{K+1}}{\argmax} \; \widetilde{T}(\delta_{1},\hdots,\delta_{j-1},\delta_{j+1},\hdots,\delta_{K},\delta_{K+1})
	\end{equation}
	satisfies $\tilde{\delta}_{K+1}=1$ (otherwise equation~\eqref{eq:8} would be true). But, since the monotonicity property is assumed to be true for problems of dimension $K$, it turns out that $\tilde{\delta}_{i}=1$ for $i=1,\hdots,j-1,j+1,\hdots,K$ as well. And, since $\tilde{\delta}_{K+1}=1$, then:
	\begin{equation}
		 \underset{\delta_{1},\hdots,\delta_{j-1},\delta_{j+1},\hdots,\delta_{K},\delta_{K+1}}{\max} \; \widetilde{T}(\delta_{1},\hdots,\delta_{j-1},\delta_{j+1},\hdots,\delta_{K},\delta_{K+1}) =  \underset{\delta_{1},\hdots,\delta_{K}}{\max} \; \frac{\left(\sum_{i=1}^{K}{\delta_{i}l_{i}(\alpha_{i})} + l_{K+1}(\alpha_{K+1}) \right)^{2}}{\sum_{i=1}^{K}{\delta_{i}\alpha_{i}l_{i}(\alpha_{i}) + \alpha_{K+1}l_{K+1}(\alpha_{K+1})}}
	\end{equation}
	which in turn implies $\hat{\delta}_{i}=1$ for $i=1,\hdots,j-1,j+1,\hdots,K$. Thus, $j$ is the only dimension which could satisfy $\hat{\delta}_{j}=0$.

To end the proof, we need to show that, indeed, it's not possible to have $\hat{\delta}_{j}=0$ either. To do so we will show that the statement of monotonicity holds for $K=3$, then we could easily show $\hat{\delta}_{j}=1$. We use a change of variables to make it clearer. 

Indeed, we rewrite $T(\delta_1,\delta_2,...,\delta_j,...,\delta_{K-1},\delta_K)$ the following way :
\[
T(\delta_1,\delta_2,...,\delta_{K-1},\delta_K)=\frac{(f_0+f_j\delta_i+f_K)^2}{\alpha_0f_0+\alpha_jf_j\delta_i+\alpha_Kf_K}
\]

with 

\[f_0=\sum_{l=1...K-1/j}f_l\] 

and 

\[f_0\alpha_0=\sum_{l=1...K-1/j}f_l\alpha_l \Leftrightarrow \alpha_0=\frac{\sum_{l=1...K-1/j}f_l\alpha_l}{\sum_{l=1...K-1/j}f_l}\]

Then, if equation~\eqref{eq:8} does \textbf{not} hold, we would have:

\begin{equation}
	\underset{\delta_{1},\hdots,\delta_{K},\delta_{K+1}}{\max} \; T(\delta_{1},\hdots,\delta_{K},\delta_{K+1}) = \underset{\delta_{0},\delta_{j},\delta_{K+1}}{\max} \; \frac{(\delta_{0}l_{0} + \delta_{j}l_{j} + \delta_{K+1}l_{K+1})^{2}}{\delta_{0}\alpha_{0}l_{0}+\delta_{j}\alpha_{j}l_{j}+\delta_{K+1}\alpha_{K+1}l_{K+1}} 
\end{equation}
where we know, by assumption, that the optimum in the right hand side is achieved when $\delta_{0}=1$ and $\delta_{K+1}=1$. If we knew monotonicity holds for K=3, it would then follow that $\delta_{j}=1$ if $\alpha_{j} \ge \alpha_{0}$. 

\end{proof}

To rephrase it, we want to show that the two following cases: $T(\delta_j=1,\delta_0=1,\delta_K=1)<T(\delta_j=0,\delta_0=1,\delta_K=1)$ with $\alpha_j<\alpha_0$ and $T(\delta_j=1,\delta_0=1,\delta_K=1)<T(\delta_j=0,\delta_0=1,\delta_K=1)$ with $\alpha_j>\alpha_0$ are impossible with the hypothesis that 
$\forall \quad \{\delta_1, \delta_2,..., \delta_{K-1}\} \quad T(\delta_1,\delta_2,...,\delta_{K-1},0)<max_{\delta_1,...,\delta_{K-1}} T(\delta_1,\delta_2,...,\delta_{K-1},1)$



First, we show that when $\alpha_j<\alpha_0$, then $T(\delta_j=1,\delta_0=1,\delta_K=1)>T(\delta_j=0,\delta_0=1,\delta_K=1)$. 
Indeed, after developing the difference we obtain :

\begin{align}\label{eq:result1}
&T(\delta_j=1,\delta_0=1,\delta_K=1)-T(\delta_j=0,\delta_0=1,\delta_K=1)\\
&= \frac{l_j}{(l_0\alpha_0+l_K \alpha_K)(l_j \alpha_j+l_0 \alpha_0+l_K\alpha_K)}(l_0^2(2\alpha_0-\alpha_j)\\
&+l_K^2 (2\alpha_K-\alpha_i)+l_0l_K(2\alpha_0+2\alpha_K-\alpha_j+l_0l_j\alpha_0+l_Kl_j\alpha_K))
&>0
\end{align}

As $\alpha_j<\alpha_0<\alpha_K$, all the terms of the previous sum are positive, which implies that $T(\delta_j=1,\delta_0=1,\delta_K=1)>T(\delta_j=0,\delta_0=1,\delta_K=1)$.

In a second time we want to show that the case $T(\delta_0=1,\delta_j=1,\delta_K=1)<T(\delta_0=1,\delta_j=0,\delta_K=1)$ with  $\alpha_j>\alpha_0$ is not possible either. In this case we use a Reductio ad absurdum: we are going to show that we can not have both $T(\delta_0=1,\delta_j=0,,1)>T(1,1,1)$ and $T(\delta_0=1,\delta_j=0,1)>T(\delta_0=0,\delta_j=1,0)$. Indeed after developing both inequalities, we find 

\begin{align}\label{eq:inequality1}
&T(\delta_0=1,\delta_j=0,1)>T(1,1,1) \Leftrightarrow \alpha_0<\frac{1}{l_0}(\alpha_j\frac{(l_0+l_K)^2}{2(l_0+l_K)+l_j}-\alpha_Kl_K)\\
 &T(\delta_0=1,\delta_j=0,1)>T(\delta_0=1,\delta_j=0,0)\Leftrightarrow \alpha_0>\frac{l_0}{2l_0+l_K}\alpha_K
\end{align}

The first inequality of \ref{eq:inequality1} can be simplified the following way, by using the following inequalities $\alpha_0<\alpha_j<\alpha_K$ and $\forall i \quad l_i>0$.

\begin{align}\label{eq:inequalitysimpl}
\alpha_0 &<\frac{1}{l_0}(\alpha_i\frac{(l_0+l_K)^2}{2(l_0+l_K)+l_j}-\alpha_Kl_K)\\
&<\frac{1}{l_0}(\alpha_K\frac{(l_0+l_K)^2}{2(l_0+l_K)+l_j}-\alpha_Kl_K)=\alpha_K(\frac{1}{l_0}\frac{(l_0+l_K)^2}{2(l_0+l_K)+l_j}-l_K)
\end{align}

Using \ref{eq:inequality1} and \ref{eq:inequalitysimpl} we have the following result :

\begin{align}\label{eq:result2}
&\frac{l_0}{2l_0+l_K}\alpha_K<\alpha_0<\alpha_K\frac{1}{l_0}(\frac{(l_0+l_K)^2}{2(l_0+l_K)+l_j}-l_K)\\
&\Rightarrow \frac{l_0}{2l_0+l_K}<\frac{1}{l_0}(\frac{(l_0+l_K)^2}{2(l_0+l_K)+l_i}-l_K)\\
&\Rightarrow 0<-(l_0+l_K)^2(l_K+l_j)
\end{align}

The last line of the previous equation set shows clearly the contradiction. 

Those two results \ref{eq:result1} and \ref{eq:result2} end the proof.

\section*{Real Data Application: Arabidopsis}

We evaluate the ability of FastCMH to detect significant contiguous patterns while correcting for confounders from an association mapping study for \textit{Arabidopsis thaliana}, a small flowering plant. Significant contiguous patterns are groups of SNPs (Single Nucleotide Polymorphisms, i.e. mutations occuring on the DNA sequence) significantly correlated with traits or phenotypes (as the size of the plant, its resistance to bacterias, the shape of the leaves...). For Arabidopsis data it is well known that geographic location is a major covariate because it is strongly related to population structure. 

\subsection*{Data description}
The data are a widely used \textit{Arabidospsis thaliana} GWAS dataset by Atwell et al. (2010) from online resource \textit{easy}GWAS (Grimm et al., 2012). This dataset is composed of 107 tables, related to a continuous or dichotomous phenotype, for at most 194 samples of inbred lines and a total of 214,051 SNPs. Out of those 107 phenotypes, there are 21 dichotomous phenotypes. A previous study found more than 2 significant patterns in 7 out of the 21 tables. For the experiments we took among them the three phenotype names YEL, LY and LES for which the population structure, measured by the genomic inflation factor $\lambda_{GI}>2$, is the biggest. All of them contain 95 sample lines and 214,051 SNPs. We clustered the sample lines according to geographical locations, a first cluster containing the samples harvested outside Europe ($N_{o}=25$), and three other European clustered. For those, we visually clustered the remaining samples using the biggest components of a Principal Component Analysis, resulting into a Western ($N_{w}=23$), a Northern ($N_{n}=20$) and a Central European ($N_{c}=27$) cluster. 

\subsection*{Experiments and results}
We ran a Chi2-test and FastCMH on each of the dataset. It appears that for those samples with a high population structure FastCMH finds less significant patterns than the Chi2-test as it is shown in Table 1 in the main text. 
To confirm this result we ran a second experiment: for each phenotype we randomized the geographical labels among the samples, 150 times, and calculated respectively the number of significant patterns. The number of patterns found by FastCMH on the original tables is each time situated in the lower tale of the distribution of the number of patterns in case of a random covariate. It shows that our method is correcting for confounders by decreasing the number of false positives.

\section*{Generating confounded significant intervals}

We now describe a procedure for generating a \emph{confounded significant interval}, as mentioned in the Experiments section in the main text. Consider the covariance matrix

\[\Sigma =  \left( \begin{array}{ccc}
1 & 0 & \rho_{\text{sig}} \\
0 & 1 & \rho_{\text{con}} \\
\rho_{\text{sig}} & \rho_{\text{con}} & 1 \end{array} \right)\] 

Consider sampling a single $z$ from the multivariate Bernoulli distribution with mean $(0.5,0.5,0.5)$ and covariance matrix $\Sigma$. This can be done in R using the \emph{bindata} package, which results in a three-vector $(x, c, y)$.

This $z$ will be three-dimensional; the first component $z_1$ is an indicator function, which indicates $(1)$ if that sequence $x_i$ will contain a significant subsequence, or not $(0)$. The second component $z_2$ is the categorical label, and the third component $z_3$ will be the class label $y_i$. 

Then, for another variable $z_4 \sim \text{B}(p_{\epsilon})$, for $p_{\epsilon}=0.1$, take $z_5 = \text{XOR} (z_2, z_4)$ to indicate if that subsequence will have a confounded significant subsequence.

\section*{Additional Experiments} 

Below are additional experiments showing the comparing the power and speed of the various algorithms described in the Experiments section in the main text.

\begin{figure}
\centering
\includegraphics[width=10cm]{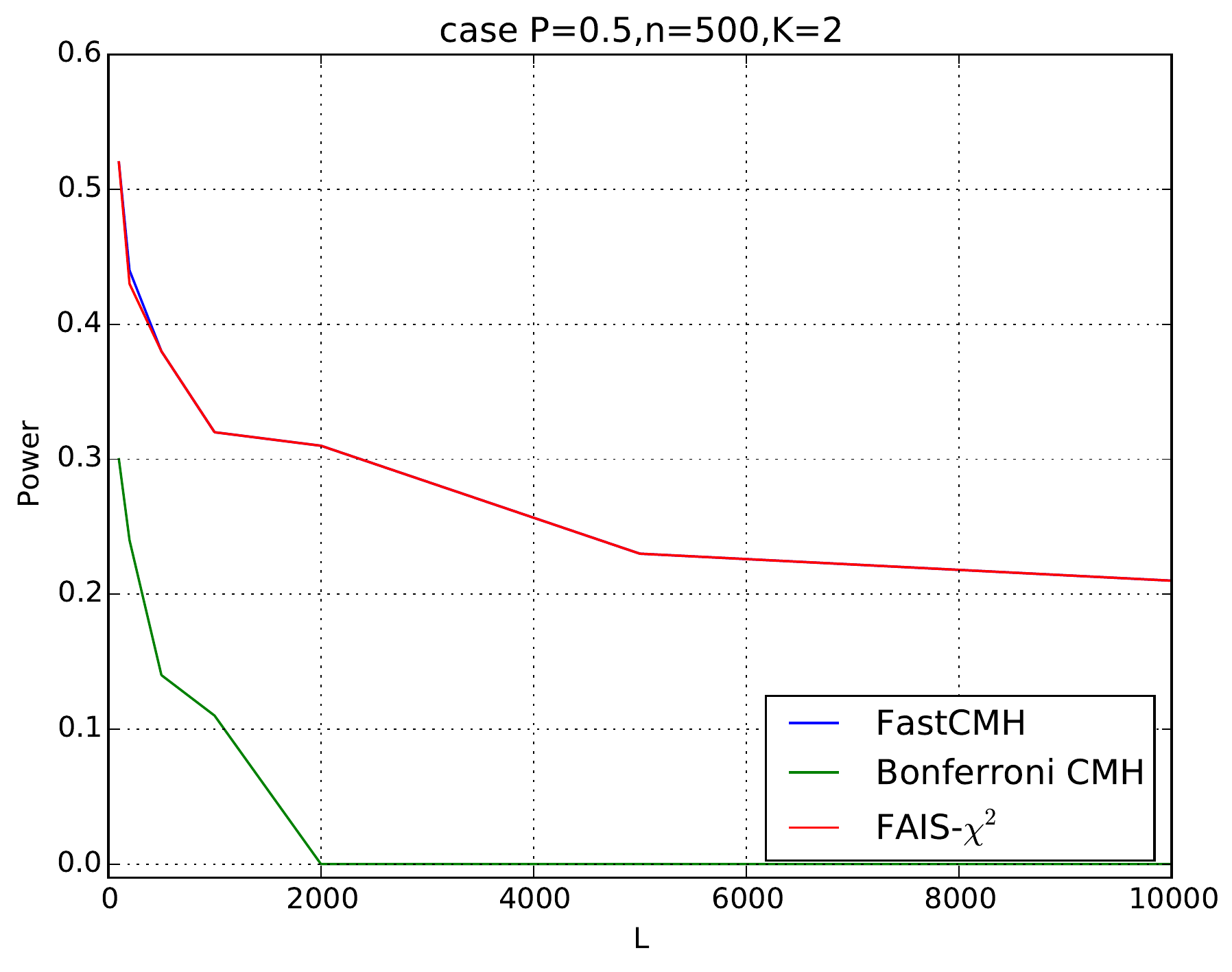}
\caption{A figure comparing the power of \texttt{FastCMH}, \texttt{FAIS-$\chi^2$} and \texttt{Bonferroni-CMH} as the length of the sequences $L$ varies.}
\label{fig:powerL}
\end{figure}

\begin{figure}
\centering
\includegraphics[width=10cm]{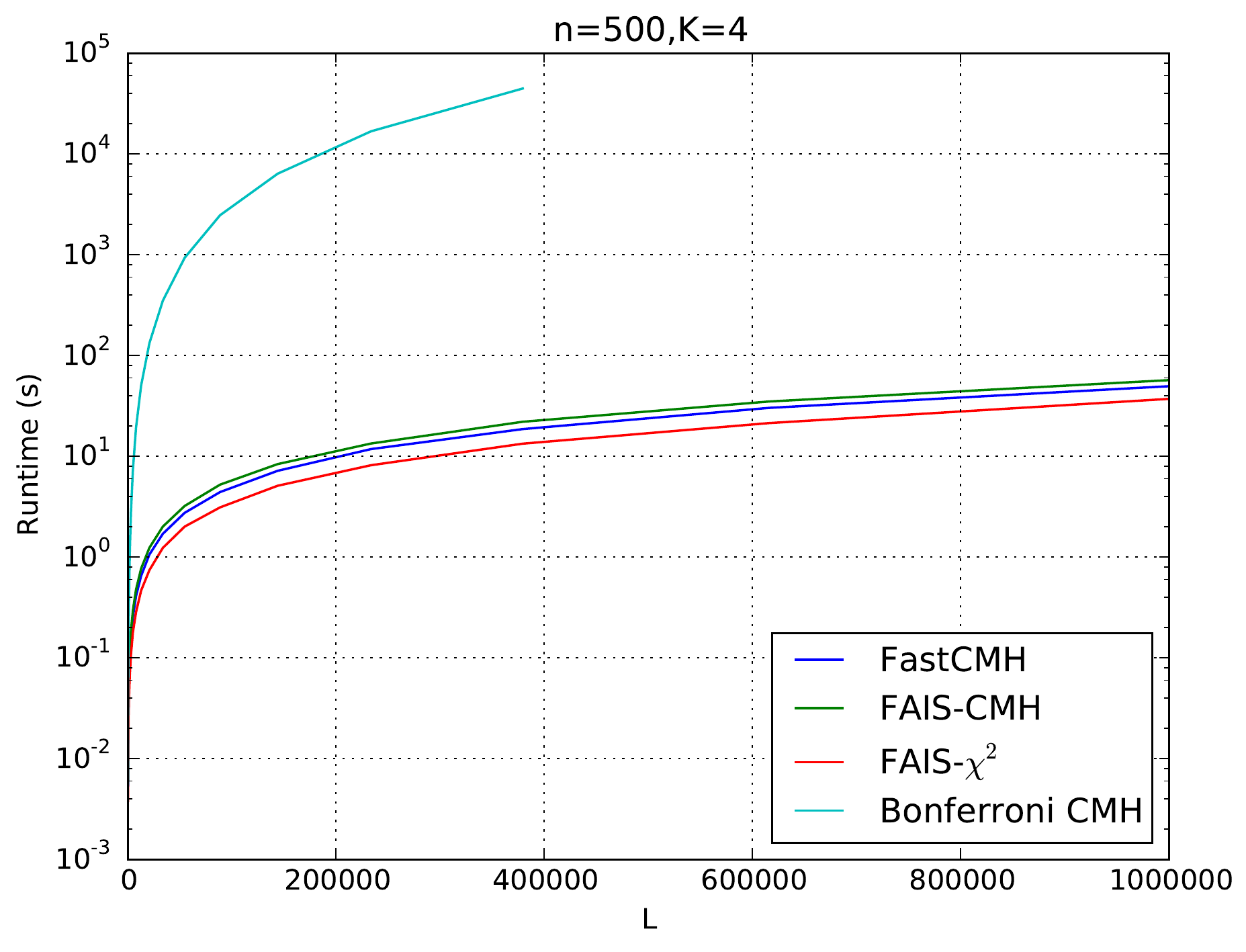}
\caption{A figure comparing the speed of \texttt{FastCMH}, \texttt{FAIS-CMH}, \texttt{FAIS-$\chi^2$} and \texttt{Bonferroni-CMH} as the length of the sequences $L$ varies, for $K=4$}
\label{fig:speedL}
\end{figure}

\begin{figure}
\centering
\includegraphics[width=10cm]{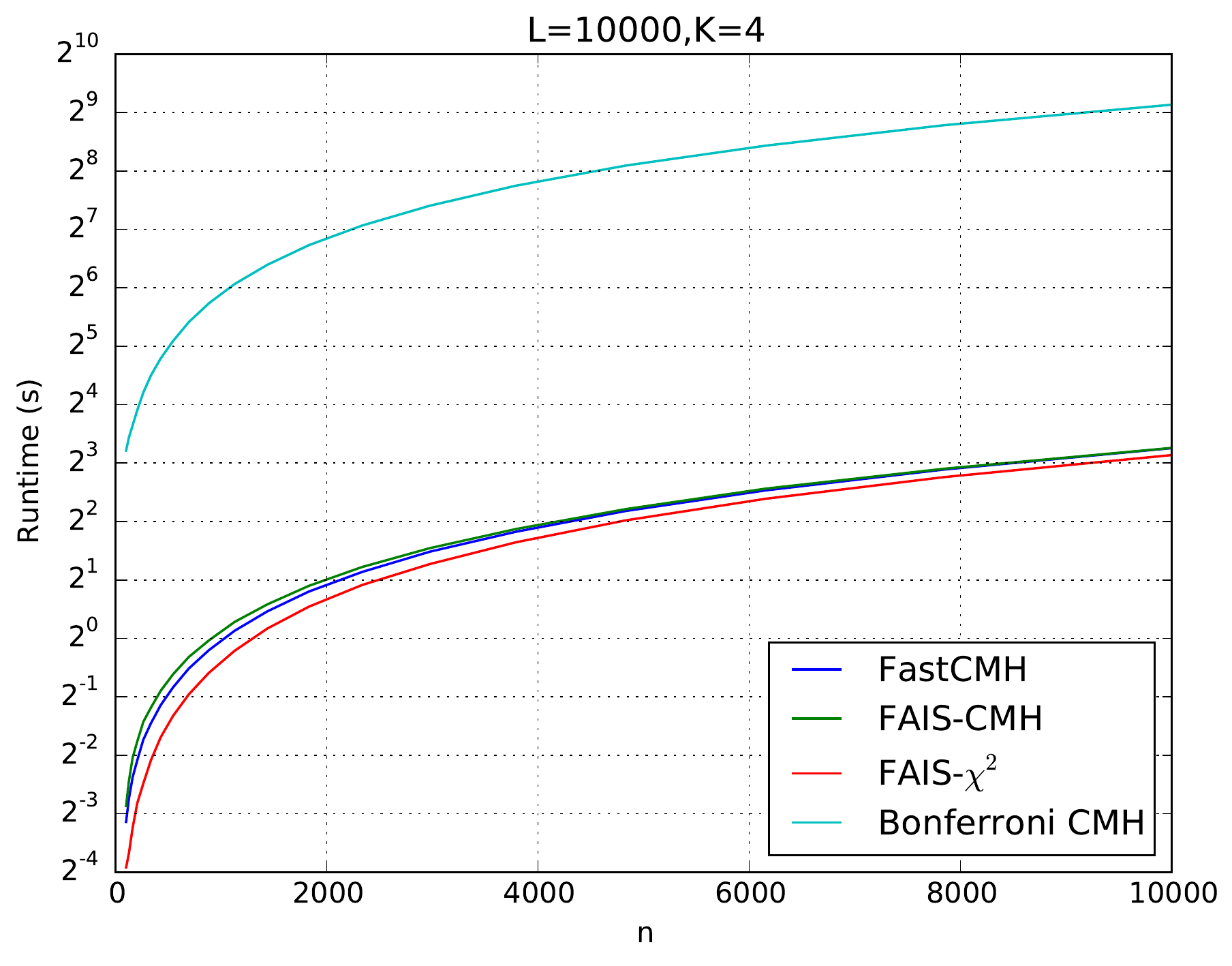}
\caption{A figure comparing the speed of \texttt{FastCMH}, \texttt{FAIS-CMH}, \texttt{FAIS-$\chi^2$} and \texttt{Bonferroni-CMH} as the number of samples $n$ varies, for $K=4$}
\label{fig:speedN}
\end{figure}

\begin{figure}
\centering
\includegraphics[width=10cm]{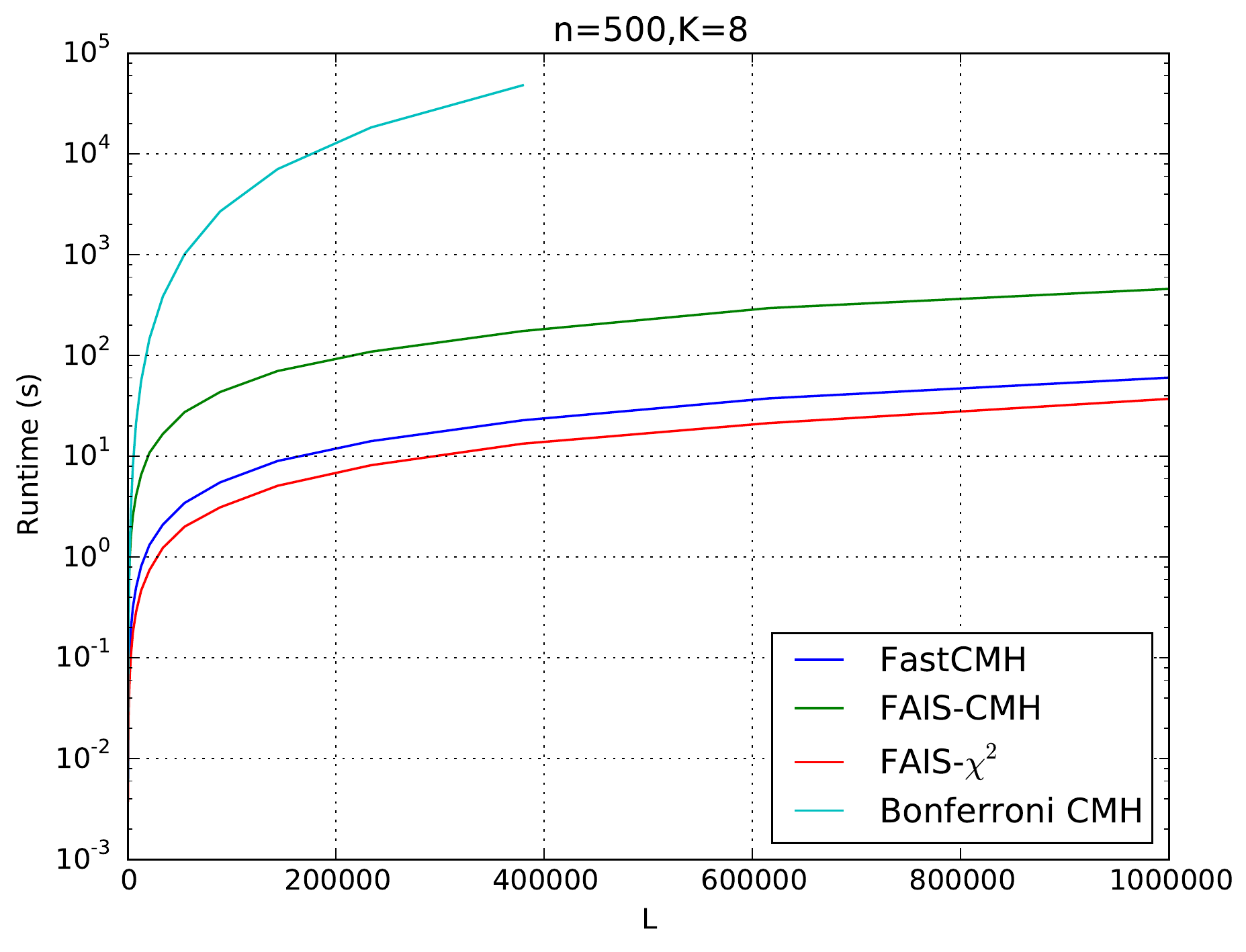}
\caption{A figure comparing the speed of \texttt{FastCMH}, \texttt{FAIS-CMH}, \texttt{FAIS-$\chi^2$} and \texttt{Bonferroni-CMH} as the length of the sequences $L$ varies, for $K=8$. Notice the increased difference between \texttt{FAIS-CMH} and \texttt{FastCMH}, when compared to Figure~\ref{fig:speedL}.}
\label{fig:speedL2}
\end{figure}

\begin{figure}
\centering
\includegraphics[width=10cm]{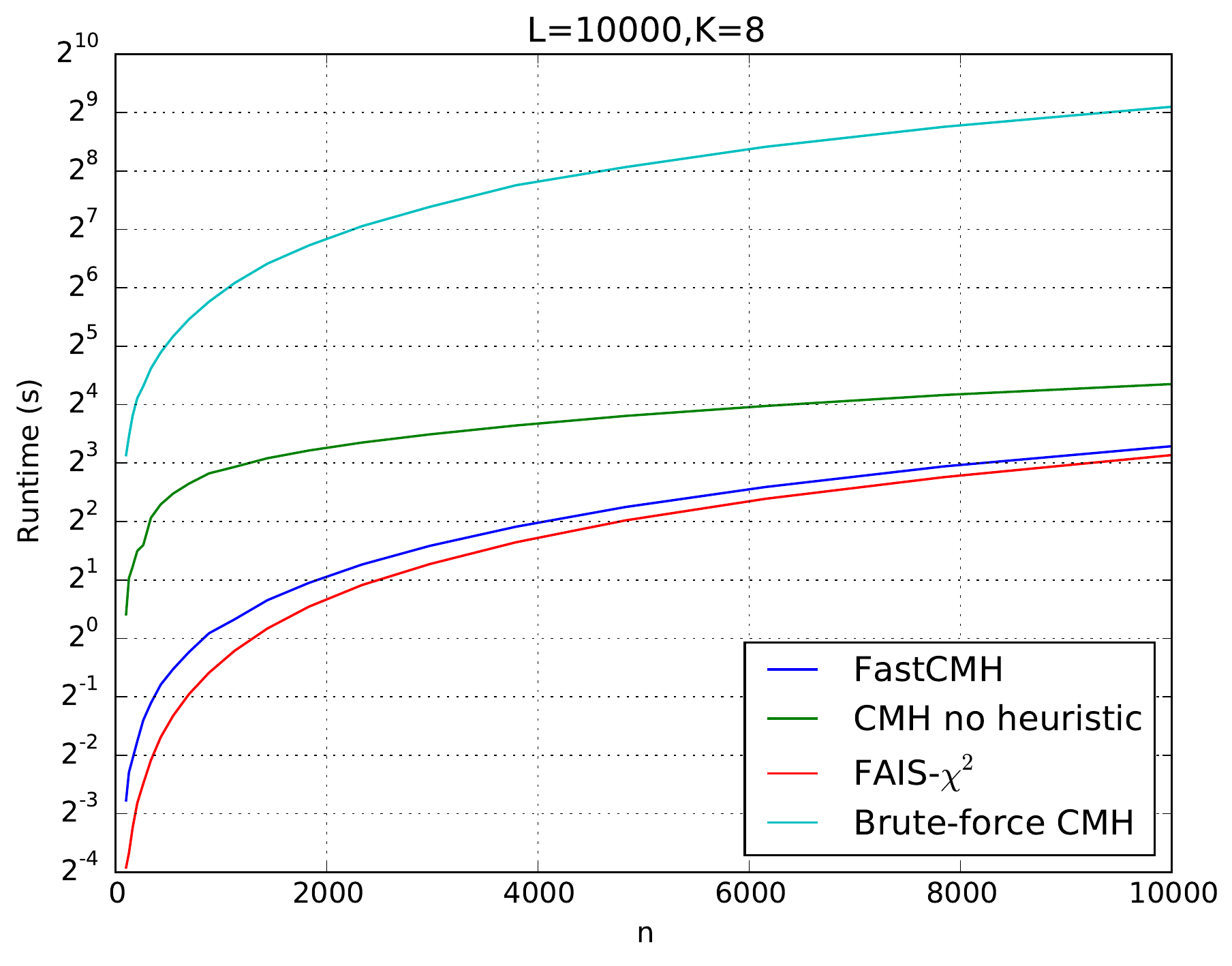}
\caption{A figure comparing the speed of \texttt{FastCMH}, \texttt{FAIS-CMH}, \texttt{FAIS-$\chi^2$} and \texttt{Bonferroni-CMH} as the number of samples $n$ varies, for $K=8$. Notice the increased difference between \texttt{FAIS-CMH} and \texttt{FastCMH}, when compared to Figure~\ref{fig:speedN}.}
\label{fig:speedN2}
\end{figure}


\newpage
\bibliographystyle{plain}
\bibliography{bibliography}

\begin{thebibliography}{10}
\bibitem{Pearson1900}
K~Pearson.
\newblock On the criterion that a given system of deviations from the probable
  in the case of a correlated system of variables is such that it can
  reasonable be supposed to have arisen from random sampling.
\newblock {\em Philosophical Magazine}, 50:157--175, 1900.
\end{thebibliography}

\end{document}